\documentclass{article}
\usepackage[utf8]{inputenc}
\usepackage{geometry}
\newgeometry{vmargin={1in}, hmargin={1.25in,1.25in}}
\usepackage[T1]{fontenc}    %
\usepackage{hyperref}       %
\usepackage{url}            %
\usepackage{booktabs}       %
\usepackage{amsfonts}       %
\usepackage{nicefrac}       %
\usepackage{microtype}      %
\usepackage{xcolor}         %

\usepackage{amsmath}
\usepackage{amssymb}
\usepackage{mathtools}
\usepackage{amsthm}

\usepackage[capitalize,noabbrev]{cleveref}

\theoremstyle{plain}
\newtheorem{theorem}{Theorem}[section]
\newtheorem{proposition}[theorem]{Proposition}
\newtheorem{lemma}[theorem]{Lemma}
\newtheorem{corollary}[theorem]{Corollary}
\theoremstyle{definition}
\newtheorem{definition}[theorem]{Definition}
\newtheorem{assumption}[theorem]{Assumption}
\theoremstyle{remark}
\newtheorem{remark}[theorem]{Remark}
\usepackage[round]{natbib}
\usepackage{amsmath,amssymb, amsthm}
\usepackage{amsfonts}
\usepackage{mathabx}
\usepackage{soul}
\usepackage{pgfplots}
\pgfplotsset{width=10cm,compat=1.9}
\usepgfplotslibrary{external}
\usepackage{amsopn}
\DeclareMathOperator*{\argmin}{argmin}

\newcommand{\eps}{\varepsilon}
\newcommand{\wt}{\widetilde}

\newcommand{\PP}{\mathcal{P}}

\usepackage{hyperref} 
\hypersetup{
	colorlinks=true,
    linkcolor=red,
    citecolor=blue,
    filecolor=black,
    urlcolor=black,
}
\usepackage{cleveref}
\crefname{algorithm}{Algorithm}{Algorithms}
\crefname{assumption}{Assumption}{Assumptions}
\crefname{equation}{}{}
\crefname{figure}{Fig.}{Figs.}
\crefname{table}{Table}{Tables}
\crefname{section}{Section}{Sections}
\crefname{subsection}{Section}{Sections}
\crefname{theorem}{Theorem}{Theorems}
\crefname{lemma}{Lemma}{Lemmmas}
\crefname{proposition}{Proposition}{Propositions}
\crefname{definition}{Definition}{Definitions}
\crefname{corollary}{Corollary}{Corollaries}
\crefname{remark}{Remark}{Remarks}
\crefname{example}{Example}{Examples}
\crefname{appendix}{Appendix}{Appendices}

\usepackage{color}
\usepackage{enumitem}
\usepackage{subfigure}
\usepackage{wrapfig}
\usepackage{svg}
\usepackage{algorithm}
\usepackage{algorithmic}
\newcommand{\XX}{\mathcal{X}}

\newcommand{\WW}{\mathcal{W}}

\newcommand{\rank}{\mbox{rank}}

\newcommand{\ws}{w^{*}}

\newcommand{\hf}{\widehat{F}_{X}}

\newcommand{\expec}{\mathbb{E}}

\newcommand{\alg}{\mathcal{A}}
\usepackage{mathtools}

\definecolor{Periwinkle}{rgb}{0.94823529, 0.31411765, 0.59411765}

\newcommand{\Al}{\mathcal{A}}
\newcommand{\wpr}{w_{\text{priv}}}

\newcommand{\pr}{\mathbb{P}}

\newcommand{\trank}{r}

\title{How to Make the Gradients Small Privately: \\Improved Rates for Differentially Private Non-Convex Optimization}

\author{Andrew Lowy\thanks{Wisconsin Institute for Discovery, University of Wisconsin-Madison. Supported by NSF award DMS-2023239 and AFOSR award FA9550-21-1-0084. \texttt{alowy@wisc.edu}} \and Jonathan Ullman\thanks{Khoury College of Computer Sciences, Northeastern University. Supported by NSF awards CNS-2232629 and CNS-2247484. 
 \texttt{jullman@ccs.neu.edu}} \and Stephen J. Wright\thanks{Department of Computer Sciences, University of Wisconsin–Madison. Supported by NSF awards DMS-2023239 and CCF-2224213, and AFOSR award FA9550-21-1-0084. \texttt{swright@cs.wisc.edu}}}

\begin{document}
\maketitle

\begin{abstract}
We provide a simple and flexible framework for designing differentially private algorithms to find approximate stationary points of non-convex loss functions.  Our framework is based on using a private approximate risk minimizer to ``warm start'' another private algorithm for finding stationary points.  We use this framework to obtain improved, and sometimes optimal, rates for several classes of non-convex loss functions.  First, we obtain improved rates for finding stationary points of smooth non-convex empirical loss functions.  Second, we specialize to quasar-convex functions, which generalize star-convex functions and arise in learning dynamical systems and training some neural nets. We  achieve the \textit{optimal} rate for this class.   
Third, we give an \textit{optimal} algorithm for finding stationary points of functions satisfying the Kurdyka-{\L}ojasiewicz (KL) condition. For example, over-parameterized neural networks often satisfy this condition. 
Fourth, we provide new state-of-the-art rates for stationary points of non-convex \textit{population} loss functions. Fifth, we obtain improved rates for non-convex generalized linear models. 
A modification of our algorithm achieves nearly the same rates for \textit{second-order} stationary points of functions with Lipschitz Hessian, improving over the previous state-of-the-art for each of the above problems. 
\end{abstract}

\section{Introduction}
The increasing prevalence of machine learning (ML) systems, such as large language models (LLMs), in societal contexts has led to growing concerns about the privacy of these models. Extensive research has demonstrated that ML models can leak the training data of individuals, violating their privacy~\citep{shokri2017membership, carlini2021extracting}. For instance, individual training examples were extracted from GPT-2 using only black-box queries~\citep{carlini2021extracting}. \textit{Differential privacy}~(DP)~\citep{dwork2006calibrating} provides a rigorous guarantee that training data cannot be leaked.  Informally, it guarantees that an adversary cannot learn much more about an individual piece of training data than they could have learned had that piece never been collected.

\vspace{.2cm}
Differentially private optimization has been studied extensively over the last 10--15 years~\citep{bst14, bft19, fkt20, asigeo, lowy2023largelip}. Despite this large body of work, certain fundamental and practically important problems remain open. In particular, for minimizing \emph{non-convex} functions, which is ubiquitous in ML applications, we have a poor understanding of the optimal rates achievable under DP.

\vspace{.2cm}
In this work, we measure the performance of an algorithm for optimizing a non-convex function $g$ by its ability to find an \textit{$\alpha$-stationary point}, meaning a point $w$ such that \[\|\nabla g(w)\| \leq \alpha.\] We want to understand the smallest $\alpha$ achievable.  There are several reasons to study stationary points. First, finding approximate global minima is intractable for general non-convex functions~\citep{murty1985}, but finding an approximate stationary point is tractable. Second, there are many important non-convex problems for which all stationary (or second-order stationary) points are global minima~(e.g. phase retrieval~\citep{sun2018geometric}, matrix completion~\citep{ge2016matrix}, and training certain classes of neural networks~\citep{liu2022loss}). Third, even for problems where it is tractable to find approximate global minima, the stationarity gap may be a better measure of quality than the excess risk~\citep{nesterov2012make,allen2018make}. 

\paragraph{Stationary Points of Empirical Loss Functions.}
A fundamental open problem in DP optimization is determining the sample complexity of finding stationary points of non-convex \textit{empirical loss functions}\[
\hf(w) := \frac{1}{n}\sum_{i=1}^n f(w, x_i),
\] 
where $X = (x_1, \ldots, x_n)$ denotes a fixed data set. For \textit{convex} loss functions, the minimax optimal complexity of DP empirical \textit{risk minimization} is $\hf(w) - \min_{w'} \hf(w') = \tilde{\Theta}(\sqrt{d \ln(1/\delta)}/\varepsilon n)$~\citep{buv14,bst14,su16}. Here $d$ is the dimension of the parameter space and $\eps,\delta$ are the privacy parameters.
However, the algorithm of \citet{bst14} was suboptimal in terms of finding DP stationary points. This gap was recently closed by~\citet{arora2022faster}, who showed that the optimal rate for stationary points of \emph{convex} $\hf$ is $\expec\|\nabla \hf(w)\| = \wt{\Theta}(\sqrt{d \ln(1/\delta)}/\eps n)$.  For \textit{non-convex} $\hf$, the best known rate prior to 2022 was $O((\sqrt{d \ln(1/\delta)}/\eps n)^{1/2})$ \citep{zhang2017efficient, wang2017ERM, wang19ncerm}. In the last two years, a pair of papers made progress and obtained improved rates of $\wt{O}((\sqrt{d \ln(1/\delta)}/\eps n)^{2/3})$~\citep{arora2022faster,tran2022momentum}. \citet{arora2022faster} gave a detailed discussion of the challenges of further improving beyond the $\wt{O}((\sqrt{d \ln(1/\delta)}/\eps n)^{2/3})$ rate. Thus, a natural question is: 
\begin{center}
\noindent\fbox{
    \parbox{0.6\linewidth}{
    \vspace{-0.cm}
\textbf{Question 1.} Can we improve the $\wt{O}((\sqrt{d \ln(1/\delta)}/\eps n)^{2/3})$ rate for DP stationary points of smooth non-convex empirical loss functions? %
}}
    \end{center}

\paragraph{Contribution 1.} We answer Question 1 affirmatively, giving a novel DP algorithm that finds a $\wt{O}((\sqrt{d \ln(1/\delta)}/\eps n)d^{1/6})$-stationary point. This rate improves over the prior state-of-the-art whenever $d < n \eps$.

\paragraph{Contribution 2.} We provide algorithms that achieve the \emph{optimal} rate $\wt{O}((\sqrt{d \ln(1/\delta)}/\eps n))$ for two subclasses of non-convex loss functions: \textit{quasar-convex} functions~\citep{hinder2020near}, which generalize \emph{star-convex} functions~\citep{nesterov2006cubic}, and \textit{Kurdyka-Łojasiewicz} (KL) functions~\citep{kurdyka1998gradients}, which generalize Polyak-{\L}ojasiewicz (PL) functions~\citep{polyak}. %
Quasar-convex functions arise in learning dynamical systems and training recurrent neural nets \citep{hardt2018gradient,hinder2020near}. Also, the loss functions of some neural networks may be quasar-convex in large neighborhoods of the minimizers~\citep{kleinberg2018alternative,zhou2019sgd}. On the other hand, the KL condition is satisfied by overparameterized neural networks in many scenarios~\citep{bassily2018exponential,liu2020toward, scaman2022convergence}. This is \textit{the first time that the optimal rate has been achieved without assuming convexity}. To the best of our knowledge, no other DP algorithm in the literature would be able to get the optimal rate for either of these function classes.

\paragraph{Second-Order Stationary Points.}
Recently, \citet{wang2021escaping,gao2023differentially,liu2023private} provided DP algorithms for finding $\alpha$-\textit{second-order stationary points} (SOSP) of functions $g$ with $\rho$-Lipschitz Hessian. A point $w$ is an $\alpha$-SOSP of $g$ if $w$ is an $\alpha$-FOSP \textit{and} \[ \nabla^2 g(w) \succeq -\sqrt{\alpha \rho}~\mathbf{I}_d. \]
The state-of-the-art rate for $\alpha$-SOSPs of empirical loss functions is due to \citet{liu2023private}: $\alpha = \wt{O}((\sqrt{d \ln(1/\delta)}/\eps n)^{2/3})$, which matches the state-of-the-art rate for FOSPs~\citep{arora2022faster,tran2022momentum}. 

\paragraph{Contribution 3.} Our framework readily extends to SOSPs and achieves an improved $\wt{O}((\sqrt{d \ln(1/\delta)}/\eps n)d^{1/6})$ second-order-stationarity guarantee.   

\paragraph{Stationary Points of Population Loss Functions.}
Moving beyond empirical loss functions, we also consider finding stationary points of \textit{population loss} functions \[
F(w) := \expec_{x \sim \PP}[f(w,x)],
\]
where $\PP$ is some unknown data distribution and we are given $n$ i.i.d. samples $X \sim \PP^n$.
The prior state-of-the-art rate for finding SOSPs of $F$ is $\wt{O}(1/n^{1/3} + (\sqrt{d}/\eps n)^{3/7})$~\citep{liu2023private}. 

\paragraph{Contribution 4.} We give an algorithm that improves over the state-of-the-art rate for SOSPs of the population loss in the regime $d < n \eps$. When $d = \Theta(1) = \eps$, our algorithm is \textit{optimal} and matches the \textit{non-private} lower bound  $\Omega(1/\sqrt{n})$. 

\vspace{.2cm}
We also specialize to (non-convex) \textit{generalized linear models} (GLMs), which have been studied privately by~\citet{song2021evading, bassily2021differentially, arora2022glm, arora2022faster,shen2023differentially}. GLMs arise, for instance, in robust regression~\citep{amid2019robust} or when 
fine-tuning the last layers of a neural network. Thus, this problem has applications in privately fine-tuning LLMs~\citep{yu2021llm,li2021large}. Denoting the rank of the design matrix $X$ by $\trank \leq \min(d, n)$, the previous state-of-the-art rate for finding FOSPs of GLMs was $O(1/\sqrt{n} + \min\{(\sqrt{\trank}/\eps n)^{2/3}, 1/(\eps n)^{2/5}\})$~\citep{arora2022faster}. 

\paragraph{Contribution 5.} We provide improved rates of finding first- and second-order stationary points of the \emph{population loss} of GLMs. Our algorithm finds a $\wt{O}(1/\sqrt{n} + \min\{(\sqrt{\trank}/\eps n)\trank^{1/6}, 1/(\eps n)^{3/7}\})$-stationary point, which is better than~\citet{arora2022faster} when $\trank < n \eps$. 

\vspace{.1cm}
A summary of our main results is given in Table~\ref{table: summary}.

\begin{figure*}
\centering
        \includegraphics[width = 0.95\textwidth]{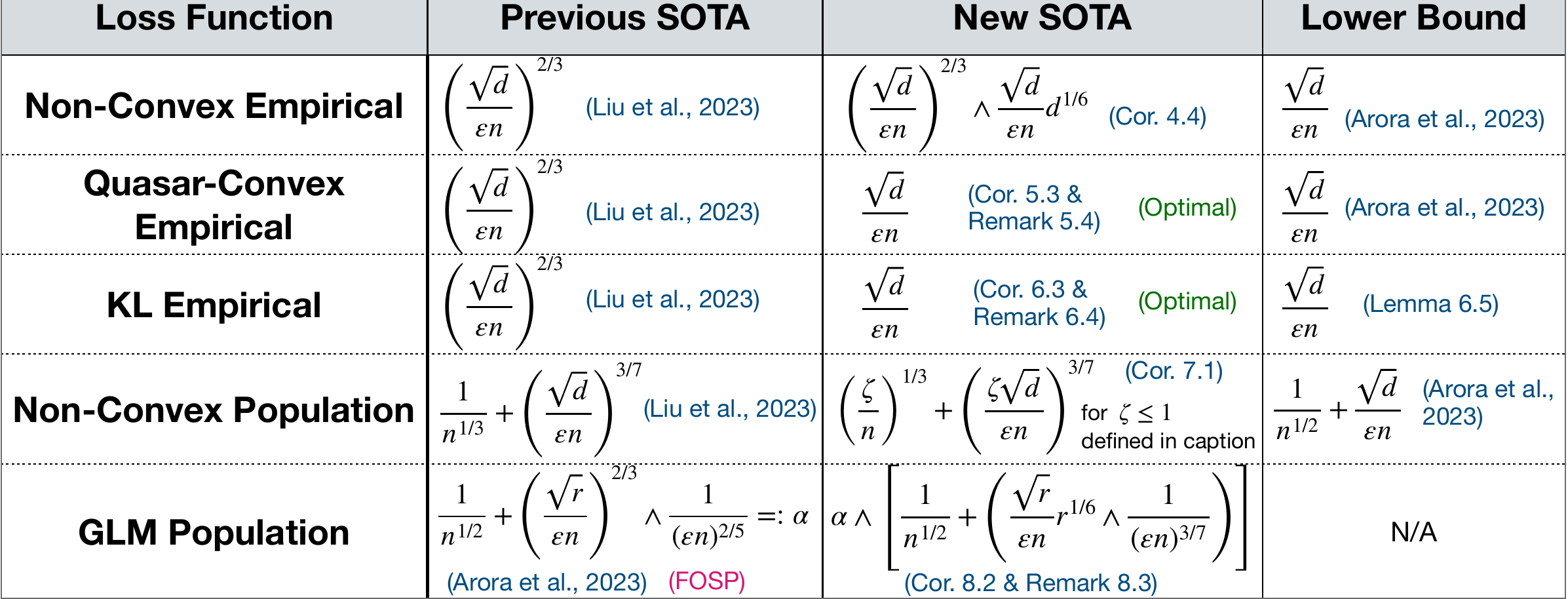}
      \vspace{-0.15in}
     \caption{\footnotesize 
Summary of results for second-order stationary points (SOSP). All bounds should be read as $\min(1, ...)$. SOTA = state-of-the-art. $\zeta := 1 \wedge \left(\frac{d}{\varepsilon n} + \sqrt{\frac{d}{n}}\right)$. $r := \rank(X)$. We omit logarithms, Lipschitz and smoothness paramaters. The GLM algorithm of \citet{arora2022faster} only finds FOSP, not SOSP. 
}\label{table: summary}
\vspace{-.1in}
\end{figure*}

\subsection{Our Approach}
Our algorithmic approach is inspired by Nesterov, who proposed the following method for finding stationary points in non-private convex optimization: first run $T$ steps of accelerated gradient descent (AGD) to obtain $w_0$, and then run $T$ steps of gradient descent (GD) initialized at $w_0$~\citep{nesterov2012make}. Nesterov's approach provided improved stationary guarantees for convex loss functions, compared to running either AGD or GD alone.

\vspace{.2cm}
We generalize and extend Nesterov's approach to private non-convex optimization. We first observe that there is nothing special about AGD or GD that makes his approach work. As we will see, one can obtain improved (DP) stationarity guarantees by running algorithm $\mathcal{B}$ after algorithm $\alg$, provided that: (a) $\alg$ moves us in the direction of a global minimizer, and (b) the stationarity guarantee of $\mathcal{B}$ benefits from a small initial suboptimality gap.  Intuitively, the algorithm $\alg$ functions as a ``warm start'' that gets us a bit closer to a global minimizer, which allows $\mathcal{B}$ to converge faster.

\subsection{Roadmap}
\cref{sec: prelims} contains relevant definitions, notations, and assumptions. In~\cref{sec: framework}, we describe our general algorithmic framework and provide privacy and stationarity guarantees. The remaining sections contain applications of our algorithmic framework to non-convex empirical losses (\cref{sec: ERM}), quasar-convex losses (\cref{sec: quasar-convex}), KL losses (\cref{sec: KL}), population losses (\cref{sec: pop loss}), and GLMs (\cref{sec: GLM}). 

\section{Preliminaries}
\label{sec: prelims}
We consider loss functions $f: \mathbb{R}^d \times \XX \to \mathbb{R}$, where $\XX$ is a data universe. For a data set $X \in \XX^n$, let $\hf(w) := \frac{1}{n}\sum_{i=1}^n f(w, x_i)$ denote the empirical loss function. Let $F(w) := \expec_{x \sim P}[f(w,x)]$ denote the population loss function with respect to some unknown data distribution $P$. 
\paragraph{Assumptions and Notation.}
\begin{definition}[Lipschitz continuity]
Function $g: \mathbb{R}^d \to \mathbb{R}$ is $L$-Lipschitz if $|g(w) - g(w')| \leq L\|w - w'\|_2$ for all $w, w' \in \mathbb{R}^d$. 
\end{definition}

\begin{definition}[Smoothness]
Function $g: \mathbb{R}^d \to \mathbb{R}$ is $\beta$-smooth if $g$ is differentiable and has $\beta$-Lipschitz gradient: $\|\nabla g(w) - \nabla g(w')\|_2 \leq \beta\|w - w'\|_2$.
\end{definition}

We assume the following throughout: 
\begin{assumption}
\label{ass: Lipschitz and smooth}
\begin{enumerate}
    \item $f(\cdot, x)$ is $L$-Lipschitz for all $x \in \XX$.
    \item $f(\cdot, x)$ is $\beta$-smooth for all $x \in \XX$.
    \item $\hf^* := \inf_{w} \hf(w) > -\infty$ for empirical loss optimization, or $F^* := \inf_w F(w) > -\infty$ for population. 
\end{enumerate}
\end{assumption}

\begin{definition}[Stationary Points]
Let $\alpha \geq 0$. We say $w$ is an $\alpha$-\textit{first-order-stationary point} (FOSP) of function $g$ if $\|\nabla g(w)\| \leq \alpha$. If the Hessian $\nabla^2 g$ is $\rho$-Lipschitz, then $w$ is an $\alpha$-\textit{second-order-stationary point} (SOSP) of $g$ if $\|\nabla g(w)\| \leq \alpha$ and $\nabla^2 g(w) \succeq -\sqrt{\rho \alpha}~\mathbf{I}_d$.
\end{definition}

\vspace{.2cm}
For functions $a = a(\theta)$ and $b = b(\phi)$ of input parameter vectors $\theta$ and $\phi$, we write $a \lesssim b$ if there is an absolute constant $C > 0$ such that $a \leq C b$ for all values of input parameter vectors $\theta$ and $\phi$. We use $\tilde{O}$ to hide logarithmic factors. Denote $a \wedge b = \min(a,b)$.

\paragraph{Differential Privacy.}
\begin{definition}[Differential Privacy~\citep{dwork2006calibrating}]
\label{def: DP}
Let $\varepsilon \geq 0, ~\delta \in [0, 1).$ A randomized algorithm $\Al: \XX^n \to \mathcal{W}$ is \textit{$(\varepsilon, \delta)$-differentially private} (DP) if for all pairs of data sets $X, X' \in \XX^n$ differing in one sample
and all measurable subsets $S \subseteq \WW$, we have
\[
\mathbb{P}(\alg(X) \in S) \leq e^\varepsilon \mathbb{P}(\alg(X') \in S) + \delta.
\]
\end{definition}

An important fact about DP is that it composes nicely:
\begin{lemma}[Basic Composition]
\label{lem: basic comp}
If $\alg$ is $(\eps_1, \delta_1)$-DP and $\mathcal{B}$ is $(\eps_2, \delta_2)$-DP, then $\mathcal{B} \circ \alg$ is $(\eps_1 + \eps_2, \delta_1 + \delta_2)$-DP. 
\end{lemma}
For a proof, see e.g., \citep{dwork2014}. There is also tighter version of the composition result---the \textit{advanced composition} theorem---which we re-state in~\cref{app: more privacy prelims}.

\section{Our Warm-Start Algorithmic Framework}
\label{sec: framework}
For ease of presentation, we will first present a concrete instantiation of our algorithmic framework for ERM, built upon the DP-SPIDER algorithm of \citet{arora2022faster}, which is described in~\cref{alg: spider}. 

\begin{algorithm}[tb]
   \caption{DP-SPIDER \citep{arora2022faster}} 
   \label{alg: spider}
\begin{algorithmic}
   \STATE {\bfseries Input:} Data $X \in \XX^n$, loss function $f(w,x)$, $(\eps, \delta)$, initialization $w_0$, stepsize $\eta$, iteration number $T$, phase length $q$, noise variances $\sigma_1^2, \sigma_2^2, \hat{\sigma}_2^2$, batch sizes $b_1, b_2$. 
   \FOR{$t=0, \ldots, T-1$}
   \IF{$q | t$}
   \STATE Sample batch $S_t$ of size $b_1$
   \STATE Sample $g_t \sim \mathcal{N}(0, \sigma_1^2 \mathbf{I}_d)$
   \STATE $\nabla_t = \frac{1}{b_1} \sum_{x \in S_t} \nabla f(w_t, x) + g_t$
   \ELSE
   \STATE Sample batch $S_t$ of size $b_2$
   \STATE Sample $h_t \sim \mathcal{N}(0, \min\{\sigma_2^2\|w_t - w_{t-1}\|^2, \hat{\sigma_2}^2\}\mathbf{I}_d)$ 
   \STATE $\Delta_t = \frac{1}{b_2}\sum_{x \in S_t}[\nabla f(w_t,x) - \nabla f(w_{t-1}, x)] + h_t$
   \STATE $\nabla_t = \nabla_{t-1} + \Delta_t$
   \ENDIF
   \STATE $w_{t+1} = w_t - \eta \nabla_t$
   \ENDFOR
   \STATE {\bfseries Return:} $\hat{w}\sim \textbf{Unif}(w_1, \ldots, w_T)$. 
\end{algorithmic}
\end{algorithm}

\vspace{.2cm}
For initialization $w_0 \in \mathbb{R}^d$, denote the suboptimality gap by \[
\hat{\Delta}_{w_0} := \hf(w_0) - \hf^*.\] We recall the guarantees of DP-SPIDER below: 
\begin{lemma}{\citep{arora2022faster}}
\label{lem: spider}
There exist algorithmic parameters such that \cref{alg: spider} is $(\eps/2, \delta/2)$-DP and returns $\hat{w}$ satisfying \begin{align}
\label{eq: spider}
\expec \| \nabla \hf(\hat{w})\| \lesssim \left(\frac{\sqrt{\hat{\Delta}_{w_0} L \beta} \sqrt{d \ln(1/\delta)}}{\eps n}\right )^{2/3} + \frac{L \sqrt{d \ln(1/\delta)}}{\eps n}. 
\end{align}
\end{lemma}
Typically, the first term on the right-hand-side of~\cref{eq: spider} is dominant. 

\vspace{.2cm}
Our algorithm is based on a simple observation: the stationarity guarantee in~\cref{lem: spider} depends on the initial suboptimality gap $\hat{\Delta}_{w_0}$. Therefore, if we can privately find a good ``warm start'' point $w_0$ such that $\hf(w_0) - \hf^*$ is small with high probability, then we can run DP-SPIDER initialized at $w_0$ to improve over the $O((\sqrt{d}/\eps n)^{2/3})$ guarantee of DP-SPIDER. More generally, we can apply any DP stationary point finder $\mathcal{B}$ with initialization $w_0$ after warm starting. 
Pseudocode for our general meta-algorithm is given in Algorithm~\ref{alg:meta}. 
\begin{algorithm}[tb]
   \caption{Warm-Start Meta-Algorithm for ERM}
   \label{alg:meta}
\begin{algorithmic}
   \STATE {\bfseries Input:} Data $X \in \XX^n$, loss function $f(w,x)$, privacy parameters $(\eps, \delta)$, warm-start DP-ERM algorithm $\alg$, DP-ERM stationary point finder $\mathcal{B}$.
   \STATE Run $(\eps/2, \delta/2)$-DP $\alg$ on $\hf(\cdot)$ to obtain $w_0$.
   \STATE Run $\mathcal{B}$ on $\hf(\cdot)$ with initialization $w_0$ and privacy parameters $(\eps/2, \delta/2)$ to obtain $\wpr$.
   \STATE {\bfseries Return:} $\wpr$. 
\end{algorithmic}
\end{algorithm}

\vspace{.2cm}
We have the following guarantee for~\cref{alg:meta} instantiated with $\mathcal{B} =$ \cref{alg: spider}. 
\begin{theorem}[First-Order Stationary Points for ERM: Meta-Algorithm]
\label{thm: meta}
Let $\zeta \leq \sqrt{d}/\eps n$. Suppose $\alg$ is $(\eps/2, \delta/2)$-DP and $\hf(\alg(X)) - \hf^* \leq \psi$ with probability $\geq 1 - \zeta$. Then, \cref{alg:meta} with $\mathcal{B}$ as DP-SPIDER is $(\eps, \delta)$-DP and returns $\wpr$ with \begin{align*}
\expec\|\nabla \hf(\wpr)\| \lesssim \frac{L \sqrt{d \ln(1/\delta)}}{\eps n} + L^{1/3} \beta^{1/3} \psi^{1/3} \left(\frac{\sqrt{d \ln(1/\delta)}}{\eps n}\right)^{2/3}. 
\end{align*}
\end{theorem}
\begin{proof}
Privacy follows from~\cref{lem: basic comp}, since $\alg$ and DP-SPIDER are both $(\eps/2, \delta/2)$-DP. 

\vspace{.2cm}
For the stationarity guarantee, let $E$ be the high-probability good event that $\hf(\alg(X)) - \hf^* \leq \psi$. Then, by \cref{lem: spider}, we have 
\begin{align*}
\expec\left[\| \nabla \hf(\wpr) \| | E \right] &\lesssim \left(\frac{\sqrt{\psi L \beta} \sqrt{d \ln(1/\delta)}}{\eps n}\right )^{2/3}  + \frac{L \sqrt{d \ln(1/\delta)}}{\eps n}.   
\end{align*}
On the other hand, if $E$ does not hold, then we still have $\| \nabla \hf(\wpr) \| \leq L$ by Lipschitz continuity. Thus, taking total expectation yields \begin{align*}
\expec\|\nabla \hf(\wpr)\| &\leq \expec\left[\| \nabla \hf(\wpr) \| | E \right](1 - \zeta) + L\zeta \\
&\lesssim \left(\frac{\sqrt{\psi L \beta} \sqrt{d \ln(1/\delta)}}{\eps n}\right )^{2/3}  + \frac{L \sqrt{d \ln(1/\delta)}}{\eps n} + L\zeta.
\end{align*}
Since $\zeta \leq \sqrt{d}/\eps n$, the result follows. 
\end{proof}

Note that if we instantiate \cref{alg:meta} with any DP $\mathcal{B}$, we can obtain an algorithm that improves over the stationarity guarantee of $\mathcal{B}$ as long as the stationarity guarantee of $\mathcal{B}$ scales with the initial suboptimality gap $\hat{\Delta}_{w_0}$. In particular, our framework allows for improved rates of finding \textit{second-order} stationarity points, by choosing $\mathcal{B}$ as the DP SOSP finder of \citet{liu2023private} (which is built on DP-SPIDER). We recall the privacy and utility guarantees of this algorithm---which we refer to as \textit{DP-SPIDER-SOSP}---below in \cref{lem: liu}. For convenience, denote \begin{align*}
\alpha&:= \left(\frac{\sqrt{\hat{\Delta}_{w_0} L \beta} \sqrt{d \ln(1/\delta)}}{\eps n}\right )^{2/3} + \frac{L \sqrt{d \ln(1/\delta)}}{\eps n} 
+ \frac{\beta}{n \sqrt{\rho}}\left(\frac{\sqrt{\hat{\Delta}_{w_0} L \beta} \sqrt{d \ln(1/\delta)}}{\eps n}\right )^{1/3}.
\end{align*} 

\begin{lemma}{\citep{liu2023private}}
\label{lem: liu}
Assume that $f(\cdot, x)$ has $\rho$-Lipschitz Hessian $\nabla^2 f(\cdot, x)$. 
Then, there is an $(\eps/2, \delta/2)$-DP Algorithm (DP-SPIDER-SOSP), that returns $\hat{w}$ such that with probability $\geq 1 - \zeta$, $\hat{w}$ is a $\wt{O}(\alpha)$-SOSP of $\hf$. 
\end{lemma}
Next, we provide the guarantee of \cref{alg:meta} instantiated with $\mathcal{B}$ as DP-SPIDER-SOSP:

\begin{theorem}[Second-order Stationary Points for ERM: Meta-Algorithm]
\label{thm: meta second order}
Suppose $\alg$ is $(\eps/2, \delta/2)$-DP and $\hf(\alg(X)) - \hf^* \leq \psi$ with probability $\geq 1 - \zeta$. Then, \cref{alg:meta} with $\mathcal{B}$ as DP-SPIDER-SOSP is $(\eps, \delta)$-DP, and with probability $\geq 1 - 2\zeta$ has output $\wpr$ satisfying

\begin{align*}
\|\nabla \hf(\wpr)\| &\leq \tilde{\alpha} := \wt{O}\left(\frac{L \sqrt{d \ln(1/\delta)}}{\eps n} + L^{1/3} \beta^{1/3} \psi^{1/3} \left(\frac{\sqrt{d \ln(1/\delta)}}{\eps n}\right)^{2/3} + \frac{\beta^{7/6} L^{1/6} \psi^{1/6}}{n \sqrt{\rho}} \left(\frac{\sqrt{d \ln(1/\delta)}}{\eps n} \right)^{1/3} \right),
\end{align*}
and \[
\nabla^2 \hf(\wpr) \succeq -\sqrt{\rho \tilde{\alpha}}~\mathbf{I}_d.\]
\end{theorem}

The proof is similar to the proof of~\cref{thm: meta}, and is deferred to Appendix~\ref{app: second order ERM meta}. 

\vspace{.2cm}
With \cref{alg:meta}, we have reduced the problem of finding an approximate stationary point $\wpr$ to finding an approximate excess risk minimizer $w_0$. 
The next question is: \textit{What should we choose as our warm-start algorithm $\alg$?}
In general, one should choose $\mathcal{A}$ that achieves the smallest possible risk for a given function class.\footnote{In particular, if there exists a DP algorithm with \textit{optimal} risk, then this algorithm is the optimal choice of warm starter.} 
In the following sections, we consider different classes of non-convex functions and instantiate \cref{alg:meta} with an appropriate warm-start $\alg$ for each class to obtain new state-of-the-art rates. 

\section{Improved Rates for Stationary Points of Non-Convex Empirical Losses}
\label{sec: ERM}
In this section, we provide improved rates for finding (first-order and second-order) stationary points of smooth non-convex empirical loss functions.
For the non-convex loss functions satisfying \cref{ass: Lipschitz and smooth}, we propose using the \textit{exponential mechanism}~\citep{mcsherry2007mechanism} as our warm-start algorithm $\alg$ in \cref{alg:meta}. 

\vspace{.2cm}
We now recall the exponential mechanism. Assume that there is a compact set $\WW \subset \mathbb{R}^d$ containing an approximate global minimizer $w^*$ such that $\hf(\ws) - \hf^* \leq LD \frac{d}{\eps n}$, and that $\|w - w'\|_2 \leq D$ for all $w, w' \in \WW$. Note that there exists a finite $D\frac{d}{\eps n}$-net for $\WW$, denoted $\wt{\WW} = \{w_1, \ldots, w_N\}$, with $N := |\wt{\WW}| \leq \left(\frac{2 D \eps n}{d} \right)^d$. In particular, $\min_{i \in [N]} \hf(w_i) - \hf^* \leq 2LD \frac{d}{\eps n}$. 

\begin{definition}[Exponential Mechanism for ERM]
\label{def: exp mech}
Given inputs $\hf, \wt{\WW}$, the exponential mechanism $\alg_E$ selects and outputs some $w \in \wt{\WW}$. The probability that a particular $w$ is selected is proportional to $\exp\left(\frac{-\eps n \hf(w)}{4 LD} \right)$.
\end{definition}

The following lemma specializes~\citep[Theorem 3.11]{dwork2014} to our ERM setting: 
\begin{lemma}
\label{lem: exp mech}
    The exponential mechanism $\alg_E$ is $\eps$-DP. Moreover, $\forall t > 0$, we have with probability at least $1 - \exp(-t)$ that
    \begin{align*}
    \hf(\alg_E) - \hf(\ws) &\leq \frac{4LD}{\eps n}\ln\left( \left(\frac{2 \eps n}{d}\right)^d + t \right) + 2LD \frac{d}{\eps n}. 
  \end{align*} 
\end{lemma}

\paragraph{First-Order Stationary Points.}
For convenience, denote \begin{equation*}
\gamma := \frac{L\sqrt{d \ln(1/\delta)}}{\eps n} + \wt{O}\left(L^{2/3}\beta^{1/3}D^{1/3}\frac{\sqrt{d \ln(1/\delta)}}{\eps n} d^{1/6}\right).
\end{equation*}

 By substituting $\eps/2$ for $\eps$ and then choosing $t = \ln(\eps n/2\sqrt{d})$ in \cref{lem: exp mech}, the $\eps/2$-exponential mechanism returns a point $w_0$ such that \begin{equation}
 \label{eq: exp mech risk}
 \hf(w_0) - \hf^* \leq 20LD\frac{d}{\eps n} \ln(\eps n/\sqrt{d}) =: \psi
  \end{equation}
 with probability at least $1 - 2\frac{\sqrt{d}}{\eps n}$. By plugging the above $\psi$ into \cref{thm: meta}, we obtain: 

\begin{corollary}[First-Order Stationary Points for Non-Convex ERM]
\label{cor: general nonconvex}
There exist algorithmic parameters such that \cref{alg:meta} with $\alg = \alg_E$ and $\mathcal{B} = $DP-SPIDER is $(\eps, \delta)$-DP and returns $\wpr$ such that \begin{align*}
\expec\|\nabla \hf(\wpr)\| &\lesssim \gamma. 
\end{align*}
\end{corollary}

If $L, \beta, D$ are constants, then \cref{cor: general nonconvex} gives $\expec\|\nabla \hf(\wpr)\| = \wt{O}\left(\frac{\sqrt{d\ln(1/\delta)}}{\eps n} d^{1/6}\right)$. This bound is bigger than the lower bound by a factor of $d^{1/6}$ and improves over the previous state-of-the-art $O\left(\frac{\sqrt{d \ln(1/\delta)}}{\eps n}\right)^{2/3}$ whenever $d < n \eps$~\citep{arora2022faster}. If $d \geq n \eps$, then one should simply run DP-SPIDER. Combining these two algorithms gives a new state-of-the-art bound for DP stationary points of non-convex empirical loss functions: \[
\expec\|\nabla \hf(\wpr)\| \lesssim \frac{\sqrt{d\ln(1/\delta)}}{\eps n} d^{1/6} \wedge \left(\frac{\sqrt{d \ln(1/\delta)}}{\eps n}\right)^{2/3}.\]

\paragraph{Challenges of Further Rate Improvements.}
We believe that it is not possible for~\cref{alg:meta} to achieve a better rate than~\cref{cor: general nonconvex} by choosing $\alg$ differently. The exponential mechanism is optimal for non-convex Lipschitz empirical risk minimization~\citep{ganesh2023universality}. Although the lower bound function in~\citet{ganesh2023universality} is not $\beta$-smooth, we believe that one can smoothly approximate it (e.g. by piecewise polynomials) to extend the same lower bound to smooth functions. For large enough $\beta$, their lower bound extends to smooth losses by simple convolution smoothing. Thus, a fundamentally different algorithm may be needed to find $O(\sqrt{d\ln(1/\delta)}/\eps n)$-stationary points for general non-convex empirical losses.

\paragraph{Second-Order Stationary Points.}
If we assume that $f$ has Lipschitz continuous Hessian, then we can instantiate \cref{alg:meta} with $\mathcal{B}$ as DP-SPIDER-SOSP to obtain: 
\begin{corollary}[Second-Order Stationary Points for Non-Convex ERM]
\label{cor: general nonconvex second order}
Let $\zeta > 0$.
Suppose $\nabla^2 f(\cdot, x)$ is $\rho$-Lipschitz~$\forall x$. Then, 
\cref{alg:meta} with $\alg = \alg_E$ and $\mathcal{B} =$ DP-SPIDER-SOSP is $(\eps, \delta)$-DP and 
with probability $\geq 1 - \zeta$, returns a $\omega$-SOSP, where  
\begin{align*}
&\omega := \gamma 
 + \wt{O}\left(\frac{L^{1/3} D^{1/6} \beta^{7/6}}{\sqrt{\rho}n}\left(\frac{\sqrt{d \ln(1/\delta)}}{\eps n}\right)^{1/2} d^{1/12}\right),
\end{align*}
\end{corollary}

If $L, \beta, D$ and $\rho$ are constants, then \cref{cor: general nonconvex second order} implies that \cref{alg:meta} finds a $\wt{O}(d^{1/6}\sqrt{d\ln(1/\delta)}/\eps n)$-\textit{second-order} stationary point of $\hf$. This result improves over the previous state-of-the-art~\citep{liu2023private} when $d<n \eps$.

\section{Optimal Rate for Quasar-Convex Losses}
\label{sec: quasar-convex}
In this section, we specialize to \textit{quasar-convex} loss functions~\citep{hardt2018gradient,hinder2020near} and show, for the first time, that it is possible to attain the optimal (up to logs) rate $\wt{O}(\sqrt{d \ln(1/\delta)}/\eps n)$ for stationary points, \textit{without assuming convexity}. 
\begin{definition}[Quasar-convex functions]
Let $q \in (0, 1]$ and let $w^*$ be a minimizer of differentiable function $g: \mathbb{R}^d \to \mathbb{R}$. $g$ is \textit{$q$-quasar convex} if for all $w \in \mathbb{R}^d$, we have \[
g(\ws) \geq g(w) +  \frac{1}{q}\langle \nabla g(w), \ws - w \rangle. 
\]
\end{definition}

Quasar-convex functions generalize star-convex functions~\citep{nesterov2006cubic}, which are quasar-convex functions with $q = 1$. Smaller values of $q < 1$ allow for a greater degree of non-convexity. 

\vspace{.2cm}
\cref{prop: quasar risk} shows that returning a uniformly random iterate of DP-SGD (\cref{alg: dp-sgd}) attains essentially the same (optimal) rate for quasar-convex ERM as for convex ERM: 

\begin{algorithm}[ht]
\caption{DP-SGD for Quasar-Convex}
\label{alg: dp-sgd}
\begin{algorithmic}[1]
\STATE {\bfseries Input:} Loss function $f$, data $X$, iteration number $T$ 
noise variance  $\sigma^2$, step size $\eta$, batch size $b$. 
\STATE Initialize $w_1 \in \mathbb{R}^d$.
 \FOR{$t \in \{1, 2, \cdots, T\}$} 
 \STATE Sample batch $S_t$ of size $b$ from $X$
\STATE Sample $u_t \sim \mathcal{N}(0, \sigma^2 \mathbf{I}_d)$
 \STATE $\nabla_t = \frac{1}{b} \sum_{x \in S_t} \nabla f(w_t, x) + u_t$ 
 \STATE $w_{t+1} = w_t - \eta \nabla_t$.
 \ENDFOR \\
\STATE {\bfseries Output:} $\hat{w} \sim \textbf{Unif}(w_1, \ldots, w_{T})$.
\end{algorithmic}
\end{algorithm}

\begin{proposition}
\label{prop: quasar risk}
Let $\hf$ be $q$-quasar convex and $\|w_{1} - \ws\| \leq D$ for $\ws \in \argmin_{w} \hf(w)$. Then, there are algorithmic parameters such that 
\cref{alg: dp-sgd} 
is $(\eps, \delta)$-DP, and returns $\hat{w}$ such that \begin{align*}
\expec \hf(\hat{w}) - \hf^* \lesssim LD \frac{\sqrt{d \ln(1/\delta)}}{\eps n q}.
\end{align*}
Further, $\forall~\zeta > 0$, there is an $(\eps, \delta)$-DP variation of \cref{alg: dp-sgd} that returns $\tilde{w}$ s.t. with probability at least $1 - \zeta$, \[
\hf(\tilde{w}) - \hf^* = \wt{O}\left(LD \frac{\sqrt{d \ln(1/\delta)}}{\eps n q}\right). 
\] 
 
\end{proposition}

See~\cref{app: quasar} for a proof. 
The same proof works for non-smooth quasar-convex losses if we replace gradients by subgradients in~\cref{alg: dp-sgd}. As a byproduct, our proof yields a novel non-private optimization result: SGD achieves the optimal $O(1/\sqrt{T})$ rate for Lipschitz non-smooth quasar-convex stochastic optimization. To our knowledge, this result was only previously recorded for smooth losses \citep{gower} or convex losses \citep{nesterov2013introductory}.

\vspace{.2cm}
By combining \cref{prop: quasar risk} with \cref{thm: meta}, we obtain: \begin{corollary}[Quasar-Convex ERM]
\label{cor: quasar stationary points}
Let $\hf$ be $q$-quasar convex and $\|w_{1} - \ws\| \leq D$ for some $w_{1} \in \mathbb{R}^d, \ws \in \argmin_{w} \hf(w)$. Then, there are algorithmic parameters such that \cref{alg:meta} with $\alg = $~\cref{alg: dp-sgd} and $\mathcal{B} = $~DP-SPIDER is $(\eps, \delta)$-DP and returns $\wpr$ such that \begin{align*}
\expec\|\nabla \hf(\wpr)\| &\lesssim L\frac{\sqrt{d \ln(1/\delta)}}{\eps n} + \wt{O}\left(L^{2/3}\beta^{1/3}D^{1/3}\frac{\sqrt{d \ln(1/\delta)}}{\eps n q}\right).
\end{align*}
\end{corollary}

If $q$ is constant and $\beta D \lesssim L$, then this rate is optimal up to a logarithmic factor, since it matches the convex (hence quasar-convex) lower bound of \citet{arora2022faster}. 

\begin{remark}
One can obtain a second-order stationary point with essentially the same (near-optimal) rate by appealing to \cref{thm: meta second order} instead of \cref{thm: meta}. 
\end{remark}

\section{Optimal Rates for KL* Empirical Losses}
\label{sec: KL}
In this section, we derive optimal rates (up to logarithms) for functions satisfying the Kurdyka-Łojasiewicz* (KL*) condition~\citep{kurdyka1998gradients}: 
\begin{definition}
\label{def: KL}
Let $\gamma, k > 0$. Function $g: \mathbb{R}^d \to \mathbb{R}$ satisfies the \textit{$(\gamma, k)$-KL*} condition on $\WW \subset \mathbb{R}^d$ if \[
g(w) - \inf_{w' \in \mathbb{R}^d} g(w') \leq \gamma^k \|\nabla g(w)\|^k, \]
for all $w \in \WW$. 
If $k = 2$ and $\gamma = \sqrt{1/2\mu}$, say $g$ satisfies the \textit{$\mu$-PL*} condition on $\WW$. 
\end{definition}

The KL* (PL*) condition relaxes the KL (PL) condition, by requiring it to only hold on a \textit{subset} of $\mathbb{R}^d$. 

\vspace{.2cm}
Near-optimal \textit{excess risk} guarantees for the KL* class were recently provided in \citep{menart}:
\begin{lemma} \citep[Theorem 1]{menart}
\label{lem: KL risk}
Assume $\hf$ satisfies the $(\gamma, k)$-KL* condition for some $k \in [1,2]$ on a centered ball $B(0, D)$ of diameter $D = \frac{\hat{\Delta}_0^{1/k}}{\gamma \beta} + \hat{\Delta}_0^{(k-1)/k} \gamma$. Then, there is an $(\eps/2, \delta/2)$-DP algorithm with output $w_0$ such that with probability at least $1 - \zeta$, \begin{align*}
&\hf(w_0) - \hf^* \leq  \wt{O}\left(\left[\frac{\gamma L \sqrt{d\ln(1/\delta)}}{\eps n} \sqrt{1 + \left(1/\hat{\Delta}_0\right)^{(2-k)/k} \gamma^2 \beta}\right]^k \right)
\end{align*}
\end{lemma}

The KL* condition implies that any approximate stationary point is an approximate excess risk minimizer, but the converse is false. The algorithm of \citet{menart} does not lead to (near-optimal) guarantees for stationary points. However, using it
as the warm-start algorithm $\alg$ in \cref{alg:meta} gives near-optimal rates for stationary points: 
\begin{corollary}[KL* ERM]
\label{cor: KL ERM}
Grant the assumptions in \cref{lem: KL risk}. Then, \cref{alg:meta} with $\alg =$ the algorithm in \cref{lem: KL risk} and $\mathcal{B} = $ DP-SPIDER is $(\eps, \delta)$-DP and returns $\wpr$ such that \begin{align*}
&\expec \|\nabla \hf(\wpr)\| \lesssim \frac{L \sqrt{d \ln(1/\delta)}}{\eps n} 
+ \wt{O}\left(\frac{\sqrt{d \ln(1/\delta)}}{\eps n}\right)^{\frac{k+2}{3}} \left(L^{k+1} \beta \gamma^{k}\right)^{\frac{1}{3}}\left(1 + \frac{(\gamma \sqrt{\beta})^{k/3}}{\hat{\Delta}_0^{\frac{2-k}{6}}}\right).
\end{align*}
In particular, if $(\gamma \sqrt{\beta})^{k/3}/\hat{\Delta}_0^{\frac{2-k}{6}} \lesssim 1$ and $\left(\frac{\beta \gamma^k}{L^{2-k}}\right)^{1/(k-1)} \lesssim n \eps/\sqrt{d \ln(1/\delta)}$, then  \[
\expec \|\nabla \hf(\wpr)\| = \wt{O}\left(\frac{L \sqrt{d \ln(1/\delta)}}{\eps n}\right). 
\]
\end{corollary}
\begin{proof}
\cref{alg:meta} is $(\eps, \delta)$-DP by \cref{thm: meta}. Further, combining \cref{thm: meta} with \cref{lem: KL risk} implies \cref{cor: KL ERM}: plug the right-hand-side of the risk bound in \cref{cor: KL ERM} for $\psi$ in \cref{thm: meta}. 
\end{proof}

 As an example: If $\hf$ is $\mu$-PL* for $\beta/\mu \lesssim (\eps n/\sqrt{d \ln(1/\delta)})$,
then our algorithm achieves $\expec \|\nabla \hf(\wpr)\| = \wt{O}(L \sqrt{d \ln(1/\delta)}/\eps n)$. 

\begin{remark}
If $L, \beta, \gamma, \hat{\Delta}_0$ are constants, then we get the same rate as \cref{cor: KL ERM} for \textit{second-order} stationary points by using \cref{alg:meta} with $\mathcal{B}$ as DP-SPIDER-SOSP instead of DP-SPIDER.     
\end{remark} 

We show next that \cref{cor: KL ERM} is optimal up to logarithms:
\begin{lemma}[Lower bound for KL*]
\label{lem: lower bound}
Let $D, L, \beta, \gamma > 0$ and $k \in (1, 2]$ such that $k = 1 + \Omega(1)$. For any $(\eps, \delta)$-DP algorithm $\mathcal{M}$, there exists a data set $X$ and $L$-Lipschitz, $\beta$-smooth $f(\cdot, x)$ that is $(\gamma, k)$-KL over $B(0, D)$ such that \[
\expec \|\nabla \hf(\mathcal{M}(X))\| = \wt{\Omega}\left(L\min\left\{1, \frac{\sqrt{d}}{\eps n}\right\}\right). 
\]   
\end{lemma}

In contrast to the excess risk setting of \cref{lem: KL risk}, larger $k$ does not allow for faster rates of stationary points. \cref{lem: lower bound} is a consequence of the KL* excess risk lower bound~\citep[Corollary 1]{menart} and \cref{def: KL}.

\section{Improved Rates for Stationary Points of Non-Convex Population Loss}
\label{sec: pop loss}
Suppose that we are given $n$ i.i.d. samples from an unknown distribution $\PP$ and our goal is to find an $\alpha$-second-order stationary point of the population loss $F(w) = \expec_{x \sim \PP}[f(w,x)]$. 
Our framework for finding DP approximate stationary points of $F$ is described in \cref{alg:meta pop}. It is a population-loss analog of the warm-start meta-\cref{alg:meta} for stationary points of $\hf$. 
\begin{algorithm}[tb]
   \caption{Warm-Start Meta-Algorithm for Pop. Loss}
   \label{alg:meta pop}
\begin{algorithmic}
   \STATE {\bfseries Input:} Data $X \in \XX^n$, loss function $f(w,x)$, privacy parameters $(\eps, \delta)$, warm-start DP risk minimization algorithm $\alg$, DP stationary point finder $\mathcal{B}$.
   \STATE Run $(\eps/2, \delta/2)$-DP $\alg$ to obtain $w_0 \approx \argmin_w F(w)$.
   \STATE Run $\mathcal{B}$ with initialization $w_0$ and privacy parameters $(\eps/2, \delta/2)$ to obtain $\wpr$.
   \STATE {\bfseries Return:} $\wpr$. 
\end{algorithmic}
\end{algorithm}

\vspace{.2cm}
We present the guarantees for \cref{alg:meta pop} with generic $\alg$ and $\mathcal{B}$ (analogous to \cref{thm: meta second order}) in~\cref{thm: population meta} in \cref{app: population}. By taking $\alg$ to be the $\eps/2$-DP exponential mechanism and $\mathcal{B}$ to be the $(\eps/2, \delta/2)$-DP-SPIDER-SOSP of \citet{liu2023private}, we obtain a new state-of-the-art rate for privately finding second-order stationary points of the population loss:

\begin{corollary}[Second-Order Stationary Points of Population Loss - Simple Version]
\label{cor: pop loss second order stationary points}
Let $nd \geq 1/\eps^2$. Assume $\nabla^2 f(\cdot, x)$ is $1$-Lipschitz and that $L, \beta$, and $D$ are constants, where $D = \|\ws\|$ for some $\ws \in \argmin_{w} F(w)$. Then, \cref{alg:meta pop} is
$(\eps, \delta)$-DP and, with probability at least $1 - \zeta$, returns a $\kappa$-second-order-stationary point, where \begin{align*}
 \kappa &\leq \wt{O}\left(\frac{1}{n^{1/3}}\left[\frac{d}{\eps n} + \sqrt{\frac{d}{n}}\right]^{1/3}
+
\left(\frac{\sqrt{d}}{\eps n}\right)^{3/7}\left[\frac{d}{\eps n} + \sqrt{\frac{d}{n}} \right]^{3/7}\right).
\end{align*}
\end{corollary}

See \cref{app: population} for a precise statement of this corollary, and the proof. The proof combines a (novel, to our knowledge) high-probability excess population risk guarantee for the exponential mechanism (\cref{lem: pop exp mech}) with \cref{thm: population meta}. 

\vspace{.2cm}
The previous state-of-the-art rate for this problem is $\wt{O}(1/n^{1/3} + (\sqrt{d}/\eps n)^{3/7})$~\citep{liu2023private}. Thus, \cref{cor: pop loss second order stationary points} strictly improves over this rate whenever $d/(\eps n) + \sqrt{d/n} < 1$. For example, if $d$ and $\eps$ are constants, then $\kappa = \wt{O}(1/\sqrt{n})$, which is \textit{optimal} and matches the \textit{non-private} lower bound of~\citet{arora2022faster}. (This lower bound holds even with the weaker \textit{first-order} stationarity measure.)
If $d > n\eps$, then one should run the algorithm of \citet{liu2023private}. Combining the two bounds results in a new state-of-the-art bound for stationary points of non-convex population loss functions.

\section{Improved Rate for Stationary Points of Non-Convex GLMs}
\label{sec: GLM}
In this section, we restrict attention to \textit{generalized linear models} (GLMs): loss functions of the form $f(w, (x,y)) = \phi_y(\langle w, x\rangle)$ for some $\phi_y: \mathbb{R}^d \to \mathbb{R}$ that is $L$-Lipschitz and $\beta$-smooth for all $y \in \mathbb{R}$. Assume that the data domain $\XX$ has bounded $\ell_2$-diameter $\|\XX\| = O(1)$ and that the design matrix $X \in \mathbb{R}^{n \times d}$ has $\trank := \rank(X)$. 

\vspace{.2cm}
\citet{arora2022glm} provided a black-box method for obtaining dimension-independent DP stationary guarantees for non-convex GLMs. Their method applies a DP Johnson-Lindenstrauss (JL) transform to the output of a DP algorithm for finding approximate stationary points of non-convex empirical loss functions. 
\begin{lemma}\citep{arora2022faster}
\label{lem: black box GLM}
Let $\mathcal{M}$ be an $(\eps, \delta)$-DP algorithm which guarantees $\expec\|\nabla \hf(\mathcal{M}(X))\| \leq g(d,n, \beta, L, D, \eps, \delta)$ and $\|\mathcal{M}(X)\| \leq poly(n, d, \beta, L, D)$ with probability at least $1 - 1/\sqrt{n}$, when run on an $L$-Lipschitz, $\beta$-smooth $\hf$ with $\|\argmin_w \hf(w)\| \leq D$. Let $k = \argmin_{j \in \mathbb{N}}\left[g(j, n, \beta, L, D, \eps, \delta/2) + \frac{L}{\sqrt{j}}\right]\wedge \trank.$ Then, the JL method, run on $L$-Lipschitz, $\beta$-smooth GLM loss $G$ with $\|\argmin_w G(w)\| \leq D$ is $(\eps, \delta)$-DP. Further, given $n$ i.i.d. samples, the method outputs $\wpr$ s.t. \begin{align*}
    \expec\|\nabla F(\wpr)\| = \wt{O}\left(\frac{L}{\sqrt{n}} + g(k, n, \beta, L, D, \eps, \delta/2) \right).
\end{align*}
\end{lemma}

\citet{arora2022glm} used \cref{lem: black box GLM} with DP-SPIDER as $\mathcal{M}$ to obtain a stationarity guarantee for non-convex GLMs: $\wt{O}\left(1/\sqrt{n} + \min\{(\sqrt{\trank}/\eps n)^{2/3}, 1/(n\eps)^{2/5}\}\right)$ when $L, \beta = O(1)$. If we apply their JL method to the output of our \cref{alg:meta}, then we obtain an improved rate: 
\begin{corollary}[Non-Convex GLMs]
\label{cor: GLM}
Let $f(w, (x,y))$ be a GLM loss function with $\beta, L, D = O(1)$.
Then, the JL method applied to the output of $\mathcal{M}=$~\cref{alg:meta} (with $\alg =$~Exponential Mechanism and $\mathcal{B}=$~DP-SPIDER) is $(\eps, \delta)$-DP and, given $n$ i.i.d. samples, outputs $\wpr$ s.t.
\begin{align*}
    \expec\|\nabla F(\wpr)\| &\leq \wt{O}\left(\frac{1}{\sqrt{n}}
   + \frac{\sqrt{\trank}}{\eps n} \trank^{1/6} \wedge \frac{1}{(\eps n)^{3/7}} \right).
\end{align*}
\end{corollary}

See \cref{app: GLM} for the proof. \cref{cor: GLM} improves over the state-of-the-art~\citep{arora2022faster} if $\trank < n \eps$. 

\begin{remark}
We can obtain essentially the same rate for \textit{second-order} stationary points by substituting DP-SPIDER-SOSP for DP-SPIDER.     
\end{remark}

\section{Preliminary Experiments}
In this section, we conduct an empirical evaluation of our algorithm as a proof of concept. We run a small simulation with a non-convex loss function and synthetic data.\footnote{Code for the experiments is available at \url{https://github.com/lowya/How-to-Make-the-Gradients-Small-Privately/tree/main}.
}

\paragraph{Loss function and data:} $f(w,x) = \frac{1}{2}\left[\|w\|^2 + \sin(\|w\|^2)\right] + x^T w$, where $x$ is drawn uniformly from $\mathbb{B}$, the unit ball in $\mathbb{R}^d$ and $\mathcal{W} = 2\mathbb{B}$. Note that $f(\cdot, x)$ is non-convex, $6$-smooth, and $5$-Lipschitz on $\mathcal{W}$. 

\paragraph{Our algorithm:} $(\eps_2, \delta/2)$-DP-SPIDER after warm-starting with $(\eps_1, \delta/2)$-DP-SGD. (Recall that this algorithm is optimal for quasar-convex functions and $\eps_1 = \eps_2 = \eps/2$.) We run $T_1$ iterations of DP-SGD and $T_2$ iterations of DP-SPIDER. $\eps_1, \eps_2, T_1$ and $T_2$ are all hyperparameters that we tune. We require $T_1 + T_2 = 50$ and $\eps_1 + \eps_2 = \eps$.

\paragraph{Baselines:} We compare against DP-SGD and DP-SPIDER, each run for $100$ iterations. We carefully tune all hyperparameters (e.g. step size and phase length). We list the hyperparameters that we used to obtain each point in the plots in Appendix~\ref{app: hyperparameters}. 

\begin{figure*}[thb]
    \centering
    \begin{minipage}[b]{0.48\textwidth}
        \centering
        \includegraphics[width=\linewidth]{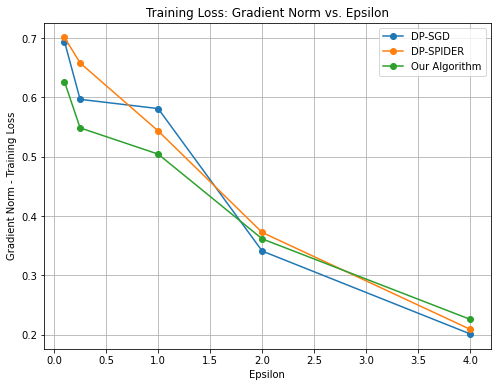}
        \vspace{-.75cm}
        \caption{Training Loss: Gradient Norm vs. $\eps$}
        \label{fig:train_loss}
    \end{minipage}
    \hfill
    \begin{minipage}[b]{0.48\textwidth}
        \centering
        \includegraphics[width=\linewidth]{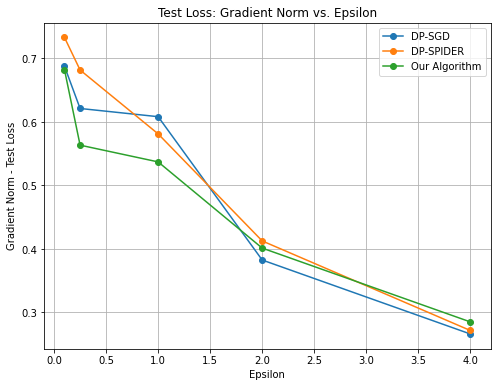}
        \vspace{-.75cm}
        \caption{Test Loss: Gradient Norm vs. $\eps$}
        \label{fig:test_loss}
    \end{minipage}
    \label{fig:grad_norms}
\end{figure*}

\paragraph{Results:} Our results are reported in Figures~\ref{fig:train_loss} and ~\ref{fig:test_loss}. \textit{Our algorithm outperforms both baselines in the high privacy regime $\eps \leq 1$.} For $\eps \in \{2, 4\}$, the performance of all $3$ algorithms is relatively similar and there is no apparent benefit from warm-starting. 

\vspace{.2cm}
Note that our algorithm generalizes both DP-SGD and DP-SPIDER: i.e., we can recover these algorithms by choosing $T_1 = 0$ or $T_2 = 0$. Thus, \textit{in theory, our algorithm should never consistently perform worse than either of these baselines}. However, due to the inherent randomness of these experiments (e.g. random Gaussian noise in each iteration), our algorithm sometimes performs worse in practice 

\paragraph{Problem parameters:} $n = d = 100,~\delta = 1/n^{1.5}$. We vary $\varepsilon \in \{0.1, 0.25, 1, 2, 4\}$. For each $\varepsilon$, we ran $10$ trials with fresh, independently drawn data and reported average results. We projected the iterates onto $\mathcal{W}$ to ensure that the smoothness and Lipschitz bounds hold in each iteration.

\section{Conclusion}
We provided a novel framework for designing private algorithms to find (first- and second-order) stationary points of non-convex (empirical and population) loss functions. Our framework led to improved rates for general non-convex loss functions and GLMs, and optimal rates for important subclasses of non-convex functions (quasar-convex and KL). 

\vspace{.2cm}
Our work opens up several interesting avenues for future exploration. First, for general non-convex empirical and population losses, there remains a gap between our improved upper bounds and the lower bounds of \citet{arora2022faster}---which hold even for \textit{convex} functions. In light of our improved upper bounds (which are optimal when $d = O(1)$), we believe that the convex lower bounds are attainable for non-convex losses. %
Second, from a practical perspective, it would be useful to understand whether improvements over the previous state-of-the-art bounds are achievable with more computationally efficient algorithms. Finally, it would be fruitful for future empirical work to have more extensive, large-scale experiments to determine the most effective way to leverage our algorithmic framework in practice. 

\section*{Acknowledgements}
AL and SW's research is supported by NSF grant 2023239 and the AFOSR award FA9550-21-1-0084. SW also acknowledges support of the NSF grant 2224213. JU's research is supported by NSF awards CNS-2232629 and CNS-2247484. AL thanks Hilal Asi for helpful discussions in the beginning phase of this project. We also thank the anonymous ICML and TPDP reviewers for their helpful feedback.

\bibliography{references}

\begin{thebibliography}{65}
\providecommand{\natexlab}[1]{#1}
\providecommand{\url}[1]{\texttt{#1}}
\expandafter\ifx\csname urlstyle\endcsname\relax
  \providecommand{\doi}[1]{doi: #1}\else
  \providecommand{\doi}{doi: \begingroup \urlstyle{rm}\Url}\fi

\bibitem[Abadi et~al.(2016)Abadi, Chu, Goodfellow, McMahan, Mironov, Talwar, and Zhang]{abadi16}
Martin Abadi, Andy Chu, Ian Goodfellow, H.~Brendan McMahan, Ilya Mironov, Kunal Talwar, and Li~Zhang.
\newblock Deep learning with differential privacy.
\newblock \emph{Proceedings of the 2016 ACM SIGSAC Conference on Computer and Communications Security}, Oct 2016.
\newblock \doi{10.1145/2976749.2978318}.
\newblock URL \url{http://dx.doi.org/10.1145/2976749.2978318}.

\bibitem[Allen-Zhu(2018)]{allen2018make}
Zeyuan Allen-Zhu.
\newblock How to make the gradients small stochastically: Even faster convex and nonconvex sgd.
\newblock \emph{Advances in Neural Information Processing Systems}, 31, 2018.

\bibitem[Amid et~al.(2019)Amid, Warmuth, Anil, and Koren]{amid2019robust}
Ehsan Amid, Manfred~KK Warmuth, Rohan Anil, and Tomer Koren.
\newblock Robust bi-tempered logistic loss based on bregman divergences.
\newblock \emph{Advances in Neural Information Processing Systems}, 32, 2019.

\bibitem[Amid et~al.(2022)Amid, Ganesh, Mathews, Ramaswamy, Song, Steinke, Suriyakumar, Thakkar, and Thakurta]{amid2022public}
Ehsan Amid, Arun Ganesh, Rajiv Mathews, Swaroop Ramaswamy, Shuang Song, Thomas Steinke, Vinith~M Suriyakumar, Om~Thakkar, and Abhradeep Thakurta.
\newblock Public data-assisted mirror descent for private model training.
\newblock In \emph{International Conference on Machine Learning}, pages 517--535. PMLR, 2022.

\bibitem[Arora et~al.(2022)Arora, Bassily, Guzm{\'a}n, Menart, and Ullah]{arora2022glm}
Raman Arora, Raef Bassily, Crist{\'o}bal Guzm{\'a}n, Michael Menart, and Enayat Ullah.
\newblock Differentially private generalized linear models revisited.
\newblock \emph{Advances in Neural Information Processing Systems}, 35:\penalty0 22505--22517, 2022.

\bibitem[Arora et~al.(2023)Arora, Bassily, Gonz{\'a}lez, Guzm{\'a}n, Menart, and Ullah]{arora2022faster}
Raman Arora, Raef Bassily, Tom{\'a}s Gonz{\'a}lez, Crist{\'o}bal~A Guzm{\'a}n, Michael Menart, and Enayat Ullah.
\newblock Faster rates of convergence to stationary points in differentially private optimization.
\newblock In \emph{International Conference on Machine Learning}, pages 1060--1092. PMLR, 2023.

\bibitem[Asi et~al.(2021)Asi, Feldman, Koren, and Talwar]{asigeo}
Hilal Asi, Vitaly Feldman, Tomer Koren, and Kunal Talwar.
\newblock Private stochastic convex optimization: Optimal rates in l1 geometry.
\newblock In Marina Meila and Tong Zhang, editors, \emph{International Conference on Machine Learning}, volume 139 of \emph{Proceedings of Machine Learning Research}, pages 393--403. PMLR, PMLR, 18--24 Jul 2021.
\newblock URL \url{https://proceedings.mlr.press/v139/asi21b.html}.

\bibitem[Asi et~al.(2023)Asi, Feldman, Koren, and Talwar]{asi2023near}
Hilal Asi, Vitaly Feldman, Tomer Koren, and Kunal Talwar.
\newblock Near-optimal algorithms for private online optimization in the realizable regime.
\newblock \emph{arXiv preprint arXiv:2302.14154}, 2023.

\bibitem[Bassily et~al.(2014)Bassily, Smith, and Thakurta]{bst14}
Raef Bassily, Adam Smith, and Abhradeep Thakurta.
\newblock Private empirical risk minimization: Efficient algorithms and tight error bounds.
\newblock In \emph{2014 IEEE 55th Annual Symposium on Foundations of Computer Science}, pages 464--473. IEEE, 2014.

\bibitem[Bassily et~al.(2018)Bassily, Belkin, and Ma]{bassily2018exponential}
Raef Bassily, Mikhail Belkin, and Siyuan Ma.
\newblock On exponential convergence of sgd in non-convex over-parametrized learning.
\newblock \emph{arXiv preprint arXiv:1811.02564}, 2018.

\bibitem[Bassily et~al.(2019)Bassily, Feldman, Talwar, and Thakurta]{bft19}
Raef Bassily, Vitaly Feldman, Kunal Talwar, and Abhradeep Thakurta.
\newblock Private stochastic convex optimization with optimal rates.
\newblock In \emph{Advances in Neural Information Processing Systems}, volume~32, 2019.

\bibitem[Bassily et~al.(2021{\natexlab{a}})Bassily, Guzm{\'a}n, and Menart]{bassily2021differentially}
Raef Bassily, Crist{\'o}bal Guzm{\'a}n, and Michael Menart.
\newblock Differentially private stochastic optimization: New results in convex and non-convex settings.
\newblock \emph{Advances in Neural Information Processing Systems}, 34:\penalty0 9317--9329, 2021{\natexlab{a}}.

\bibitem[Bassily et~al.(2021{\natexlab{b}})Bassily, Guzm{\'a}n, and Nandi]{bassily2021non}
Raef Bassily, Crist{\'o}bal Guzm{\'a}n, and Anupama Nandi.
\newblock Non-euclidean differentially private stochastic convex optimization.
\newblock In \emph{Conference on Learning Theory}, pages 474--499. PMLR, 2021{\natexlab{b}}.

\bibitem[Boob and Guzm{\'a}n(2023)]{boob2023optimal}
Digvijay Boob and Crist{\'o}bal Guzm{\'a}n.
\newblock Optimal algorithms for differentially private stochastic monotone variational inequalities and saddle-point problems.
\newblock \emph{Mathematical Programming}, pages 1--43, 2023.

\bibitem[Bun et~al.(2014)Bun, Ullman, and Vadhan]{buv14}
Mark Bun, Jonathan Ullman, and Salil Vadhan.
\newblock Fingerprinting codes and the price of approximate differential privacy.
\newblock In \emph{Proceedings of the forty-sixth annual ACM symposium on Theory of computing}, pages 1--10, 2014.

\bibitem[Carlini et~al.(2021)Carlini, Tramer, Wallace, Jagielski, Herbert-Voss, Lee, Roberts, Brown, Song, Erlingsson, et~al.]{carlini2021extracting}
Nicholas Carlini, Florian Tramer, Eric Wallace, Matthew Jagielski, Ariel Herbert-Voss, Katherine Lee, Adam Roberts, Tom~B Brown, Dawn Song, Ulfar Erlingsson, et~al.
\newblock Extracting training data from large language models.
\newblock In \emph{USENIX Security Symposium}, volume~6, pages 2633--2650, 2021.

\bibitem[Chaudhuri et~al.(2011)Chaudhuri, Monteleoni, and Sarwate]{chaud}
Kamalika Chaudhuri, Claire Monteleoni, and Anand~D Sarwate.
\newblock Differentially private empirical risk minimization.
\newblock \emph{Journal of Machine Learning Research}, 12\penalty0 (3), 2011.

\bibitem[Dwork and Roth(2014)]{dwork2014}
Cynthia Dwork and Aaron Roth.
\newblock \emph{The Algorithmic Foundations of Differential Privacy}, volume~9.
\newblock Now Publishers, Inc., 2014.

\bibitem[Dwork et~al.(2006)Dwork, McSherry, Nissim, and Smith]{dwork2006calibrating}
Cynthia Dwork, Frank McSherry, Kobbi Nissim, and Adam Smith.
\newblock Calibrating noise to sensitivity in private data analysis.
\newblock In \emph{Theory of cryptography conference}, pages 265--284. Springer, 2006.

\bibitem[Feldman et~al.(2020)Feldman, Koren, and Talwar]{fkt20}
Vitaly Feldman, Tomer Koren, and Kunal Talwar.
\newblock Private stochastic convex optimization: optimal rates in linear time.
\newblock In \emph{Proceedings of the 52nd Annual ACM SIGACT Symposium on Theory of Computing}, pages 439--449, 2020.

\bibitem[Foster et~al.(2018)Foster, Sekhari, and Sridharan]{foster2018uniform}
Dylan~J Foster, Ayush Sekhari, and Karthik Sridharan.
\newblock Uniform convergence of gradients for non-convex learning and optimization.
\newblock \emph{Advances in Neural Information Processing Systems}, 31, 2018.

\bibitem[Ganesh et~al.(2023)Ganesh, Thakurta, and Upadhyay]{ganesh2023universality}
Arun Ganesh, Abhradeep Thakurta, and Jalaj Upadhyay.
\newblock Universality of langevin diffusion for private optimization, with applications to sampling from rashomon sets.
\newblock In \emph{The Thirty Sixth Annual Conference on Learning Theory}, pages 1730--1773. PMLR, 2023.

\bibitem[Gao and Wright(2023)]{gao2023differentially}
Changyu Gao and Stephen~J Wright.
\newblock Differentially private optimization for smooth nonconvex erm.
\newblock \emph{arXiv preprint arXiv:2302.04972}, 2023.

\bibitem[Ge et~al.(2016)Ge, Lee, and Ma]{ge2016matrix}
Rong Ge, Jason~D Lee, and Tengyu Ma.
\newblock Matrix completion has no spurious local minimum.
\newblock \emph{Advances in neural information processing systems}, 29, 2016.

\bibitem[Gower et~al.(2021)Gower, Sebbouh, and Loizou]{gower}
Robert Gower, Othmane Sebbouh, and Nicolas Loizou.
\newblock Sgd for structured nonconvex functions: Learning rates, minibatching and interpolation.
\newblock In \emph{International Conference on Artificial Intelligence and Statistics}, pages 1315--1323. PMLR, 2021.

\bibitem[Hardt et~al.(2018)Hardt, Ma, and Recht]{hardt2018gradient}
Moritz Hardt, Tengyu Ma, and Benjamin Recht.
\newblock Gradient descent learns linear dynamical systems.
\newblock \emph{The Journal of Machine Learning Research}, 19\penalty0 (1):\penalty0 1025--1068, 2018.

\bibitem[Hinder et~al.(2020)Hinder, Sidford, and Sohoni]{hinder2020near}
Oliver Hinder, Aaron Sidford, and Nimit Sohoni.
\newblock Near-optimal methods for minimizing star-convex functions and beyond.
\newblock In \emph{Conference on learning theory}, pages 1894--1938. PMLR, 2020.

\bibitem[Jain and Thakurta(2014)]{jain2014near}
Prateek Jain and Abhradeep~Guha Thakurta.
\newblock (near) dimension independent risk bounds for differentially private learning.
\newblock In \emph{International Conference on Machine Learning}, pages 476--484. PMLR, 2014.

\bibitem[Kang et~al.(2021)Kang, Liu, Niu, and Wang]{kang2021weighted}
Yilin Kang, Yong Liu, Ben Niu, and Weiping Wang.
\newblock Weighted distributed differential privacy erm: Convex and non-convex.
\newblock \emph{Computers \& Security}, 106:\penalty0 102275, 2021.

\bibitem[Karimi et~al.(2016)Karimi, Nutini, and Schmidt]{karimi2016linear}
Hamed Karimi, Julie Nutini, and Mark Schmidt.
\newblock Linear convergence of gradient and proximal-gradient methods under the polyak-{\l}ojasiewicz condition.
\newblock In \emph{Joint European Conference on Machine Learning and Knowledge Discovery in Databases}, pages 795--811. Springer, 2016.

\bibitem[Kleinberg et~al.(2018)Kleinberg, Li, and Yuan]{kleinberg2018alternative}
Bobby Kleinberg, Yuanzhi Li, and Yang Yuan.
\newblock An alternative view: When does sgd escape local minima?
\newblock In \emph{International Conference on Machine Learning}, pages 2698--2707, 2018.

\bibitem[Kurdyka(1998)]{kurdyka1998gradients}
Krzysztof Kurdyka.
\newblock On gradients of functions definable in o-minimal structures.
\newblock In \emph{Annales de l'institut Fourier}, volume~48, pages 769--783, 1998.

\bibitem[Li et~al.(2021)Li, Tramer, Liang, and Hashimoto]{li2021large}
Xuechen Li, Florian Tramer, Percy Liang, and Tatsunori Hashimoto.
\newblock Large language models can be strong differentially private learners.
\newblock \emph{arXiv preprint arXiv:2110.05679}, 2021.

\bibitem[Liu et~al.(2020)Liu, Zhu, and Belkin]{liu2020toward}
Chaoyue Liu, Libin Zhu, and Mikhail Belkin.
\newblock Toward a theory of optimization for over-parameterized systems of non-linear equations: the lessons of deep learning.
\newblock \emph{arXiv preprint arXiv:2003.00307}, 7, 2020.

\bibitem[Liu et~al.(2022)Liu, Zhu, and Belkin]{liu2022loss}
Chaoyue Liu, Libin Zhu, and Mikhail Belkin.
\newblock Loss landscapes and optimization in over-parameterized non-linear systems and neural networks.
\newblock \emph{Applied and Computational Harmonic Analysis}, 59:\penalty0 85--116, 2022.

\bibitem[Liu et~al.(2023)Liu, Ganesh, Oh, and Thakurta]{liu2023private}
Daogao Liu, Arun Ganesh, Sewoong Oh, and Abhradeep~Guha Thakurta.
\newblock Private (stochastic) non-convex optimization revisited: Second-order stationary points and excess risks.
\newblock In \emph{Thirty-seventh Conference on Neural Information Processing Systems}, 2023.

\bibitem[Lowy and Razaviyayn(2023{\natexlab{a}})]{lowy2021fl}
Andrew Lowy and Meisam Razaviyayn.
\newblock Private federated learning without a trusted server: Optimal algorithms for convex losses.
\newblock In \emph{The Eleventh International Conference on Learning Representations}, 2023{\natexlab{a}}.
\newblock URL \url{https://openreview.net/forum?id=TVY6GoURrw}.

\bibitem[Lowy and Razaviyayn(2023{\natexlab{b}})]{lowy2023largelip}
Andrew Lowy and Meisam Razaviyayn.
\newblock Private stochastic optimization with large worst-case lipschitz parameter: Optimal rates for (non-smooth) convex losses and extension to non-convex losses.
\newblock In \emph{International Conference on Algorithmic Learning Theory}, pages 986--1054. PMLR, 2023{\natexlab{b}}.

\bibitem[Lowy et~al.(2023{\natexlab{a}})Lowy, Ghafelebashi, and Razaviyayn]{lgr23}
Andrew Lowy, Ali Ghafelebashi, and Meisam Razaviyayn.
\newblock Private non-convex federated learning without a trusted server.
\newblock In \emph{Proceedings of the 26th International Conference on Artificial Intelligence and Statistics (AISTATS)}, pages 5749--5786. PMLR, 2023{\natexlab{a}}.

\bibitem[Lowy et~al.(2023{\natexlab{b}})Lowy, Gupta, and Razaviyayn]{lowy2022DPfair}
Andrew Lowy, Devansh Gupta, and Meisam Razaviyayn.
\newblock Stochastic differentially private and fair learning.
\newblock In \emph{The Eleventh International Conference on Learning Representations}, 2023{\natexlab{b}}.
\newblock URL \url{https://openreview.net/forum?id=3nM5uhPlfv6}.

\bibitem[Lowy et~al.(2023{\natexlab{c}})Lowy, Li, Huang, and Razaviyayn]{lowy23public}
Andrew Lowy, Zeman Li, Tianjian Huang, and Meisam Razaviyayn.
\newblock Optimal differentially private model training with public data.
\newblock \emph{arXiv preprint:2306.15056}, 2023{\natexlab{c}}.

\bibitem[McSherry and Talwar(2007)]{mcsherry2007mechanism}
Frank McSherry and Kunal Talwar.
\newblock Mechanism design via differential privacy.
\newblock In \emph{48th Annual IEEE Symposium on Foundations of Computer Science (FOCS'07)}, pages 94--103. IEEE, 2007.

\bibitem[Mei et~al.(2018)Mei, Bai, and Montanari]{mei2018landscape}
Song Mei, Yu~Bai, and Andrea Montanari.
\newblock The landscape of empirical risk for nonconvex losses.
\newblock \emph{The Annals of Statistics}, 46\penalty0 (6A):\penalty0 2747--2774, 2018.

\bibitem[Menart et~al.(2023)Menart, Ullah, Arora, Bassily, and Guzm{\'a}n]{menart}
Michael Menart, Enayat Ullah, Raman Arora, Raef Bassily, and Crist{\'o}bal Guzm{\'a}n.
\newblock Differentially private non-convex optimization under the kl condition with optimal rates.
\newblock \emph{arXiv preprint arXiv:2311.13447}, 2023.

\bibitem[Murty and Kabadi(1985)]{murty1985}
Katta~G Murty and Santosh~N Kabadi.
\newblock Some np-complete problems in quadratic and nonlinear programming.
\newblock 1985.

\bibitem[Nesterov(2012)]{nesterov2012make}
Yurii Nesterov.
\newblock How to make the gradients small.
\newblock \emph{Optima. Mathematical Optimization Society Newsletter}, \penalty0 (88):\penalty0 10--11, 2012.

\bibitem[Nesterov(2013)]{nesterov2013introductory}
Yurii Nesterov.
\newblock \emph{Introductory lectures on convex optimization: A basic course}, volume~87.
\newblock Springer Science \& Business Media, 2013.

\bibitem[Nesterov and Polyak(2006)]{nesterov2006cubic}
Yurii Nesterov and Boris~T Polyak.
\newblock Cubic regularization of newton method and its global performance.
\newblock \emph{Mathematical Programming}, 108\penalty0 (1):\penalty0 177--205, 2006.

\bibitem[Polyak(1963)]{polyak}
Boris~T Polyak.
\newblock Gradient methods for the minimisation of functionals.
\newblock \emph{USSR Computational Mathematics and Mathematical Physics}, 3\penalty0 (4):\penalty0 864--878, 1963.

\bibitem[Scaman et~al.(2022)Scaman, Malherbe, and Dos~Santos]{scaman2022convergence}
Kevin Scaman, Cedric Malherbe, and Ludovic Dos~Santos.
\newblock Convergence rates of non-convex stochastic gradient descent under a generic lojasiewicz condition and local smoothness.
\newblock In \emph{International Conference on Machine Learning}, pages 19310--19327. PMLR, 2022.

\bibitem[Shen et~al.(2023)Shen, Wang, Xiang, Ying, and Wang]{shen2023differentially}
Hanpu Shen, Cheng-Long Wang, Zihang Xiang, Yiming Ying, and Di~Wang.
\newblock Differentially private non-convex learning for multi-layer neural networks.
\newblock \emph{arXiv preprint arXiv:2310.08425}, 2023.

\bibitem[Shokri et~al.(2017)Shokri, Stronati, Song, and Shmatikov]{shokri2017membership}
Reza Shokri, Marco Stronati, Congzheng Song, and Vitaly Shmatikov.
\newblock Membership inference attacks against machine learning models.
\newblock In \emph{2017 IEEE symposium on security and privacy (SP)}, pages 3--18. IEEE, 2017.

\bibitem[Song et~al.(2021)Song, Steinke, Thakkar, and Thakurta]{song2021evading}
Shuang Song, Thomas Steinke, Om~Thakkar, and Abhradeep Thakurta.
\newblock Evading the curse of dimensionality in unconstrained private glms.
\newblock In \emph{International Conference on Artificial Intelligence and Statistics}, pages 2638--2646. PMLR, 2021.

\bibitem[Steinke and Ullman(2016)]{su16}
Thomas Steinke and Jonathan Ullman.
\newblock Between pure and approximate differential privacy.
\newblock \emph{Journal of Privacy and Confidentiality}, 7\penalty0 (2), 2016.

\bibitem[Sun et~al.(2018)Sun, Qu, and Wright]{sun2018geometric}
Ju~Sun, Qing Qu, and John Wright.
\newblock A geometric analysis of phase retrieval.
\newblock \emph{Foundations of Computational Mathematics}, 18:\penalty0 1131--1198, 2018.

\bibitem[Tran and Cutkosky(2022)]{tran2022momentum}
Hoang Tran and Ashok Cutkosky.
\newblock Momentum aggregation for private non-convex erm.
\newblock \emph{Advances in Neural Information Processing Systems}, 35:\penalty0 10996--11008, 2022.

\bibitem[Wang and Xu(2021)]{wang2021escaping}
Di~Wang and Jinhui Xu.
\newblock Escaping saddle points of empirical risk privately and scalably via dp-trust region method.
\newblock In \emph{Machine Learning and Knowledge Discovery in Databases: European Conference, ECML PKDD 2020, Ghent, Belgium, September 14--18, 2020, Proceedings, Part III}, pages 90--106. Springer, 2021.

\bibitem[Wang et~al.(2017)Wang, Ye, and Xu]{wang2017ERM}
Di~Wang, Minwei Ye, and Jinhui Xu.
\newblock Differentially private empirical risk minimization revisited: Faster and more general.
\newblock \emph{Advances in Neural Information Processing Systems}, 30, 2017.

\bibitem[Wang et~al.(2019)Wang, Chen, and Xu]{wang19ncerm}
Di~Wang, Changyou Chen, and Jinhui Xu.
\newblock Differentially private empirical risk minimization with non-convex loss functions.
\newblock In Kamalika Chaudhuri and Ruslan Salakhutdinov, editors, \emph{Proceedings of the 36th International Conference on Machine Learning}, volume~97 of \emph{Proceedings of Machine Learning Research}, pages 6526--6535. PMLR, 09--15 Jun 2019.
\newblock URL \url{https://proceedings.mlr.press/v97/wang19c.html}.

\bibitem[Yu et~al.(2021)Yu, Naik, Backurs, Gopi, Inan, Kamath, Kulkarni, Lee, Manoel, Wutschitz, et~al.]{yu2021llm}
Da~Yu, Saurabh Naik, Arturs Backurs, Sivakanth Gopi, Huseyin~A Inan, Gautam Kamath, Janardhan Kulkarni, Yin~Tat Lee, Andre Manoel, Lukas Wutschitz, et~al.
\newblock Differentially private fine-tuning of language models.
\newblock \emph{arXiv preprint arXiv:2110.06500}, 2021.

\bibitem[Zhang et~al.(2017)Zhang, Zheng, Mou, and Wang]{zhang2017efficient}
Jiaqi Zhang, Kai Zheng, Wenlong Mou, and Liwei Wang.
\newblock Efficient private erm for smooth objectives, 2017.

\bibitem[Zhang et~al.(2022)Zhang, Thekumparampil, Oh, and He]{zhang2022bring}
Liang Zhang, Kiran~K Thekumparampil, Sewoong Oh, and Niao He.
\newblock Bring your own algorithm for optimal differentially private stochastic minimax optimization.
\newblock \emph{Advances in Neural Information Processing Systems}, 35:\penalty0 35174--35187, 2022.

\bibitem[Zhang et~al.(2021)Zhang, Ma, Lou, and Xiong]{zhangijcai}
Qiuchen Zhang, Jing Ma, Jian Lou, and Li~Xiong.
\newblock Private stochastic non-convex optimization with improved utility rates.
\newblock In \emph{Proceedings of the Thirtieth International Joint Conference on Artificial Intelligence (IJCAI-21)}, 2021.

\bibitem[Zhou et~al.(2019)Zhou, Yang, Zhang, Liang, and Tarokh]{zhou2019sgd}
Yi~Zhou, Junjie Yang, Huishuai Zhang, Yingbin Liang, and Vahid Tarokh.
\newblock Sgd converges to global minimum in deep learning via star-convex path.
\newblock \emph{arXiv preprint arXiv:1901.00451}, 2019.

\bibitem[Zhou et~al.(2020)Zhou, Chen, Hong, Wu, and Banerjee]{zhou2020private}
Yingxue Zhou, Xiangyi Chen, Mingyi Hong, Zhiwei~Steven Wu, and Arindam Banerjee.
\newblock Private stochastic non-convex optimization: Adaptive algorithms and tighter generalization bounds.
\newblock \emph{arXiv preprint arXiv:2006.13501}, 2020.

\end{thebibliography}
\bibliographystyle{plainnat}
\newpage
\appendix
\section*{Appendix}
\section{Further Discussion of Related Work}
\label{app: related work}
Private ERM and stochastic optimization with convex loss functions has been studied extensively~\citep{chaud,bst14,bft19,fkt20}. Beyond these classical settings, differentially private optimization has also recently been studied e.g., in the context of online learning~\citep{jain2014near,asi2023near}, federated learning~\citep{lowy2021fl}, different geometries~\citep{bassily2021non,asigeo}, 
min-max games~\citep{boob2023optimal,zhang2022bring}, fair and private learning~\citep{lowy2022DPfair}, and public-data assisted private optimization~\citep{amid2022public,lowy23public}. Below we summarize the literature on DP \textit{non-convex} optimization.

\paragraph{Stationary Points of Empirical Loss Functions.}
  For non-convex $\hf$, the best known stationarity rate prior to 2022 was $\expec\|\nabla \hf(\alg(X))\| = O((\sqrt{d \ln(1/\delta)}/\eps n)^{1/2})$ \citep{zhang2017efficient, wang2017ERM, wang19ncerm}. In the last two years, a pair of papers made progress and obtained improved rates of $\wt{O}((\sqrt{d \ln(1/\delta)}/\eps n)^{2/3})$~\citep{arora2022faster,tran2022momentum}. The work of \citet{lgr23} extended this result to non-convex federated learning/distributed ERM and non-smooth loss functions. The work of \citet{liu2023private} extended this result to \textit{second-order} stationary points. Despite this problem receiving much attention from researchers, it remained unclear whether the $\wt{O}((\sqrt{d \ln(1/\delta)}/\eps n)^{2/3})$ barrier could be broken. Our algorithm finally breaks this barrier.

\paragraph{Stationary Points of Population Loss Functions.}
The literature on stationary points of population loss functions is much sparser than for empirical loss functions. The work of~\citet{zhou2020private} gave a DP algorithm for finding $\alpha$-FOSP, where $\alpha \lesssim \eps \sqrt{d} + (\sqrt{d}/\eps n)^{1/2}$. Thus, their bound is meaningful only when $\eps \ll 1/\sqrt{d}$. \citet{arora2022glm} improved over this rate, obtaining $\alpha = \wt{O}(1/n^{1/3} + (\sqrt{d}/\eps n)^{1/2})$. The prior state-of-the-art rate for finding SOSPs of $F$ was $\wt{O}(1/n^{1/3} + (\sqrt{d}/\eps n)^{3/7})$~\citep{liu2023private}. We improve over this rate in the present work. %

\paragraph{Excess Risk of PL and KL Loss Functions.}
Private optimization of PL loss functions has been considered in~\citep{wang2017ERM, kang2021weighted, zhangijcai, lgr23}. Prior to the work of \citet{lgr23}, all works on DP PL optimization made the extremely strong assumptions that $f(\cdot, x)$ is Lipschitz and PL on all of $\mathbb{R}^d$. We are not aware of any loss functions that satisfy both these assumptions. This gap was addressed by~\citet{lgr23}, who proved near-optimal excess risk bounds for \textit{proximal-PL}~\citep{karimi2016linear} loss functions. The proximal-PL condition extends the PL condition to the constrained setting, and allows for functions that are Lipschitz on some compact subset of $\mathbb{R}^d$. The work of \citet{menart} gave near-optimal excess risk bounds under the KL* condition, which generalizes the PL condition. Our work is the first to give optimal bounds for finding approximate stationary points of KL* functions. Note that stationarity is a stronger measure of suboptimality than excess risk for KL* functions, since by definition, the excess risk of these functions is upper bounded by a function of the gradient norm.

\paragraph{Non-Convex GLMs.}
While DP excess risk guarantees for convex GLMs are well understood~\citep{jain2014near,song2021evading,arora2022glm}, far less is known for stationary points of non-convex GLMs.
In fact, we are aware of only one prior work that provides DP stationarity guarantees for non-convex GLMs: \citet{arora2022faster} obtains dimension-independent/rank-dependent $\alpha$-FOSP, where $\alpha \lesssim 1/\sqrt{n} + (\sqrt{r}/\eps n)^{2/3} \wedge (1/\eps n)^{2/5}$ and $r$ is the rank of the design matrix $X$. We improve over this rate in the present work. 

Non-privately, non-convex GLMs have been studied by~\citet{mei2018landscape,foster2018uniform}.

\section{More privacy preliminaries}
\label{app: more privacy prelims}

The following result can be found, e.g. in~\citet[Theorem 3.20]{dwork2014}.
\begin{lemma}[Advanced Composition Theorem]
\label{thm: advanced composition}
Let $\epsilon \geq\ 0, \delta, \delta' \in [0, 1)$. Assume $\mathcal{A}_1, \cdots, \mathcal{A}_T$, with $\Al_t: \XX^n \times \WW \to \WW$, are each $(\epsilon, \delta)$-DP ~$\forall t = 1, \cdots, T$. Then, the adaptive composition $\Al(X) := \Al_T(X, \Al_{T-1}(X, \Al_{T-2}(X, \cdots)))$ is $(\epsilon', T\delta + \delta')$-DP for $\epsilon' = \sqrt{2T \ln(1/\delta')} \epsilon + T\epsilon(e^{\epsilon} - 1)$.
\end{lemma}

\section{Second-Order Stationary Points for ERM: Meta-Algorithm}
\label{app: second order ERM meta}

\begin{theorem}[Re-statement of \cref{thm: meta second order}]
Suppose $\alg$ is $(\eps/2, \delta/2)$-DP and $\hf(\alg(X)) - \hf^* \leq \psi$ with probability $\geq 1 - \zeta$ (for polynomial $1/\zeta$). Then, \cref{alg:meta} with $\mathcal{B}$ as DP-SPIDER-SOSP (with appropriate parameters) is $(\eps, \delta)$-DP, and with probability $\geq 1 - 2\zeta$ has output $\wpr$ satisfying \begin{align*}
\|\nabla \hf(\wpr)\| &\leq \tilde{\alpha} := \wt{O}\left(\frac{L \sqrt{d \ln(1/\delta)}}{\eps n}\right) \\
\;\;\;\;&+ \wt{O}\left(L^{1/3} \beta^{1/3} \psi^{1/3} \left(\frac{\sqrt{d \ln(1/\delta)}}{\eps n}\right)^{2/3}\right) \\
\;\;\;\;&+ \wt{O}\left(\frac{\beta^{7/6} L^{1/6} \psi^{1/6}}{n \sqrt{\rho}} \left(\frac{\sqrt{d \ln(1/\delta)}}{\eps n} \right)^{1/3} \right),
\end{align*}
and \[
\nabla^2 \hf(\wpr) \succeq -\sqrt{\rho \tilde{\alpha}}~\mathbf{I}_d\].
\end{theorem}
\begin{proof}
    Let $E$ be the good event that $\hf(\alg(X)) - \hf^* \leq \psi$ \textit{and} $\mathcal{B}$ satisfies the stationarity guarantees in \cref{lem: liu} given input $w_0 = \alg(X)$. Then $\mathbb{P}(E) \geq 1 - 2 \zeta$ by a union bound. Moreover, conditional on $E$, the stationarity guarantees in \cref{thm: meta second order} hold by applying \cref{lem: liu} with parameter $\hat{\Delta}_{w_0}$ replaced by $\psi$. 
\end{proof}

\section{Optimal Rate for Quasar-Convex Losses}
\label{app: quasar}

\begin{proposition}[Precise Statement of \cref{prop: quasar risk}]
Let $\hf$ be $q$-quasar convex and $\|w_{1} - \ws\| \leq D$ for some $w_{1} \in \mathbb{R}^d, \ws \in \argmin_{w} \hf(w)$. Then, \cref{alg: dp-sgd} with 
\[
\eta = \frac{D}{\sqrt{T(L^2 + d \sigma^2)}}, \quad
T = \frac{\eps^2 n^2}{d \ln(1/\delta)}, \quad
b \gtrsim \sqrt{d \eps}, \quad \sigma^2 = \frac{1000 L^2 T \ln(1/\delta)}{\eps^2 n^2}
\]
is $(\eps, \delta)$-DP, and returns $\hat{w}$ such that \begin{align*}
\expec \hf(\hat{w}) - \hf^* \lesssim LD \frac{\sqrt{d \ln(1/\delta)}}{\eps n q}.
\end{align*}
Moreover, for any $\zeta > 0$, there is an $(\eps, \delta)$-DP variation of \cref{alg: dp-sgd} that returns $\tilde{w}$ such that \[
\hf(\hat{w}) - \hf^* = \wt{O}\left(LD \frac{\sqrt{d \ln(1/\delta)}}{\eps n q}\right) 
\] 
with probability at least $1 - \zeta$. 
\end{proposition}
\begin{proof}
\textbf{Privacy:} Privacy of DP-SGD does not require convexity and is an immediate consequence of, e.g. \citep[Theorem 1]{abadi16} and our choices of $T, b, \sigma^2$. 

\paragraph{Expected excess risk:} Recall that the updates are given by $w_{t+1} = w_t - \eta \nabla_t$, where $\nabla_t := g_t + u_t := \frac{1}{b}\sum_{x \in S_t} \nabla f(w_t, x) + u_t$ for $u_t \sim \mathcal{N}(0, \sigma^2 \mathbf{I}_d)$ and $S_t$ is drawn uniformly with replacement from $X$ with $b = |S_t|$. Thus, \begin{align*}
    \|w_{t+1} - \ws\|^2 = \|w_t - \ws\|^2 - 2\eta \langle \nabla_t, w_t - \ws \rangle + \eta^2\|\nabla_t\|^2. 
\end{align*}
Taking conditional expectation given $w_t$ and using the fact that $u_t$ is mean-zero and independent of $w_t$ gives:
\begin{align*}
    \expec\left[\|w_{t+1} - \ws\|^2 | w_t\right] &= \|w_t - \ws\|^2 - 2\eta \langle \nabla \hf(w_t), w_t - \ws \rangle + \eta^2\left( \|g_t\|^2  + d\sigma^2 \right)\\
    &\leq \|w_t - \ws\|^2 - 2\eta \langle \hf(w_t), w_t - \ws \rangle + \eta^2\left( L^2  + d\sigma^2 \right) \\
    &\leq \|w_t - \ws\|^2 - 2\eta q \left(\hf(w_t) - \hf^* \right) + \eta^2\left( L^2  + d\sigma^2 \right),
\end{align*}
where the last inequality above used $q$-quasar-convexity. Now, re-arranging and taking total expectation yields:
\begin{align*}
2 \eta q \expec[\hf(w_t) - \hf^*] \leq \expec\left[\|w_t - \ws\|^2 - \|w_{t+1} - \ws\|^2 \right] + \eta^2\left( L^2  + d\sigma^2 \right). 
\end{align*}

Telescoping the above inequality from $t=1$ to $T$ and recalling $\hat{w}_T \sim \textbf{Unif}(\{w_1, \ldots, w_T\})$ 
yields \begin{align*}
\expec[\hf(\hat{w}_T) - \hf^*] \leq \frac{D^2}{2 \eta q T} + \frac{\eta(L^2 + d\sigma^2)}{2 q}.
\end{align*}
Plugging in $\eta = \frac{D}{\sqrt{T(L^2 + d\sigma^2)}}$ then gives 
\begin{align*}
\expec[\hf(\hat{w}_T) - \hf^*] &\leq \frac{2D}{q \sqrt{T}}\left(L + \sqrt{d \sigma^2}\right)
\lesssim LD \left( \frac{1}{q \sqrt{T}} + \frac{\sqrt{d \ln(1/\delta)}}{\eps n q} \right).
\end{align*}
Finally, choosing $T \geq \frac{\eps^2 n^2}{d \ln(1/\delta)}$ yields the desired expected excess risk bound. 

\paragraph{High-probability excess risk:} This is an instantiation of the meta-algorithm described in \citep[Appendix D]{bst14}. We run the DP-SGD algorithm above $k = \log(2/\zeta)$ times with privacy parameters $(\eps/2k, \delta/2k)$ for each run. This gives us an $(\eps/2, \delta/2)$-DP list of $k$ vectors, which we denote $\{\hat{w}^1, \ldots, \hat{w}^k\}$. By Markov's inequality, with probability at least $1 - 1/2^k$, there exists $i \in [k]$ such that $\hf(\hat{w}^i) - \hf^* \lesssim \frac{LD k \sqrt{d \ln(k/\delta)}}{\eps n}$. Now we apply the $\eps/2$-DP exponential mechanism \citep{mcsherry2007mechanism} to the list $\{\hat{w}^1, \ldots, \hat{w}^k\}$ in order to select the (approximately) best $\hat{w}^i$ with probability at least $1 - \zeta/2$. By a union bound, the output of this mechanism has excess risk bounded by $\wt{O}\left(LD \frac{\sqrt{d \ln(1/\delta)}}{q \eps n} \right)$ with probability at least $1 - \zeta$. 
\end{proof}

\section{Improved Rates for Stationary Points of Non-Convex Population Loss}
\label{app: population}
Denote the initial suboptimality gap of the population loss by \[
\Delta_{w_0} := F(w_0) - F^*. 
\]
We will need the population stationary guarantees of a variation of DP-SPIDER-SOSP: 

\begin{lemma}\citep[Theorem 4.6]{liu2023private}
\label{lem: liu pop}
Let $\zeta \in (0,1)$ and let $\nabla^2 f(\cdot,x)$ be $\rho$-Lipschitz for all $x$. Denote \begin{align*}
s := \wt{O}\left(\left(\frac{L \beta \Delta_{w_0}}{n}\right)^{1/3} + (L \beta^3 \Delta_{w_0}^3)^{1/7} \left(\frac{\sqrt{d \ln(1/\delta)}}{\eps n}\right)^{3/7}\right),  
\end{align*}
and \begin{align*}
    S := \wt{O}\left(s + \frac{\beta}{\sqrt{\rho}}\left(\frac{1}{n \eps} + \frac{1}{\sqrt{n}}\right) \sqrt{s} + L\left(\frac{1}{n \eps} + \frac{1}{\sqrt{n}}\right)\right).
\end{align*}
Then, there is a $(\eps/2, \delta/2)$-DP variation of DP-SPIDER-SOSP which, given $n$ i.i.d. samples from $\PP$, returns a point $\hat{w}$ such that $\hat{w}$ is an $S$-second-order-stationary point of $F$ with probability at least $1 - \zeta$. 
\end{lemma}

\begin{theorem}[Second-Order Stationary Points for Population Loss: Meta-Algorithm]
\label{thm: population meta}
Let $\zeta \in (0,1)$ and let $\nabla^2 f(\cdot,x)$ be $\rho$-Lipschitz for all $x$. Suppose $\alg$ is $(\eps/2, \delta/2)$-DP and $F(\alg(X)) - F^* \leq \psi$ with probability $\geq 1 - \zeta$. Then, \cref{alg:meta pop} with $\mathcal{B}$ as DP-SPIDER-SOSP (with appropriate parameters) is $(\eps, \delta)$-DP and, given $n$ i.i.d. samples from $\mathcal{P}$, has output $\wpr$ which is a $\upsilon$-second-order-stationary point of $F$ with probability at least $1 - 2\zeta$, where
\begin{align*}
\upsilon &:= \wt{O}\left(\left(\frac{L \beta \psi}{n}\right)^{1/3} + (L \psi^3 \beta^3)^{1/7} \left(\frac{\sqrt{d \ln(1/\delta)}}{\eps n}\right)^{3/7}\right)  \\
&\;\;\; + \wt{O}\left(\frac{\beta}{\sqrt{\rho}}\left(\frac{1}{n \eps} + \frac{1}{\sqrt{n}} \right)\left(\frac{(L\beta \psi)^{1/6}}{n^{1/6}} + (L\psi^3 \beta^3)^{1/14} \left(\frac{\sqrt{d \ln(1/\delta)}}{\eps n}\right)^{3/14}\right)\right) \\
&\;\;\; + \wt{O}\left(L\left(\frac{1}{n \eps} + \frac{1}{\sqrt{n}} \right)\right).
\end{align*}
\end{theorem}
\begin{proof}
Privacy is immediate from basic composition. 

By assumption, $\alg$ returns $w_0$ such that $\Delta_{w_0} \leq \psi$ with probability at least $1 - \zeta$. Conditional on this good event happening, then \cref{lem: liu pop} implies the desired stationarity guarantee with probability at least $1 - \zeta$, by plugging in $\psi$ for $\Delta_{w_0}$ in \cref{lem: liu pop}. By a union bound,
we obtain \cref{thm: meta second order}.
\end{proof}

In order to obtain \cref{cor: pop loss second order stationary points}, we will also need a high-probability excess population risk guarantee for the exponential mechanism: 
\begin{lemma}[Excess Population Risk of Exponential Mechanism]
\label{lem: pop exp mech}
Let $\zeta \in (0,1)$ and let $\WW$ be a compact set containing $\tilde{w}$ such that $\|w - \tilde{w}\| \leq D$ for all $w \in \WW$ and $F(\tilde{w}) - F^* \leq LDd/\eps n$. Then, given $n$ i.i.d. samples from $\mathcal{P}$, the $\eps$-DP exponential mechanism  of \cref{def: exp mech} outputs $w_0$ such that, with probability at least $1 - \zeta$, \[
F(w_0) - F^* = \wt{O}\left(LD \left(\frac{d}{\eps n} + \sqrt{\frac{d}{n}} \right) \right). 
\]
\end{lemma}
\begin{proof}
Let $\wt{\WW} = \{w_1, \ldots, w_{N}\}$ be a $D\frac{d}{\eps n}$-net for $\WW$ with cardinality $N = |\wt{\WW}| \leq \left(\frac{2 D \eps n}{d} \right)^d$. Denote the output of the exponential mechanism $w_0 = \alg_E(X)$. By \cref{lem: exp mech}, we have
\begin{equation}
\hf(w_0) - \hf^* \leq \wt{O}\left(LD \frac{d}{\eps n} \right)
\end{equation}
with probability at least $1 - \zeta/2$. Now, for any $j \in [N]$, we have \[
\pr(|\hf(w_j) - F(w_j)| \leq p) \geq 1 - 2\exp\left(\frac{-np^2}{2L^2 D^2}\right)
\]
for any $p \in (0,1)$ by Hoeffding's inequality, since $f(w_j, x) \in [-LD, LD]$ for all $x$. By a union bound, we have \begin{equation}
\pr\left(\max_{j \in [N]} |\hf(w_j) - F(w_j)| \leq p\right) \geq 1 - 2N\exp\left(\frac{-np^2}{2L^2 D^2}\right). 
\end{equation}
Thus, the following inequalities hold with probability at least $1 - 4N\exp\left(\frac{-np^2}{2L^2 D^2}\right) - \zeta/2$:\begin{align*}
    F(w_0) - F^* &\leq \hf(w_0) - F^* + p\\
    &\leq \hf(w_0) - \hf\left(\argmin_w F(w)\right) + 2p \\
    &\leq \hf(w_0) - \hf^* + 2p\\
    &\leq \wt{O}\left(LD \frac{d}{\eps n} \right) + 2p.
\end{align*}
Choosing $p = \frac{LD}{\sqrt{n}}\sqrt{\log(8/\zeta) + d}$ ensures that \[
F(w_0) - F^* = \wt{O}\left(LD \left(\frac{d}{\eps n} + \sqrt{\frac{d}{n}} \right) \right). 
\] 
with probability at least $1 - \zeta$, as desired. 
\end{proof}
Note that \citep[Theorem 5.8]{liu2023private} proved a weaker ``in-expectation'' version of \cref{lem: pop exp mech}. 

\begin{corollary}[Precise Statement of \cref{cor: pop loss second order stationary points}]
\label{cor: appendix pop loss second order}
Assume $\nabla^2 f(\cdot, x)$ is $\rho$-Lipschitz and $\WW$ is a compact set containing $\tilde{w}$ such that $\|w - \tilde{w}\| \leq D$ for all $w \in \WW$ and $F(\tilde{w}) - F^* \leq LDd/\eps n$. Then, given $n$ i.i.d. samples from $\mathcal{P}$, \cref{alg:meta pop} with $\alg =$ Exponential Mechanism and $\mathcal{B}$ = DP-SPIDER-SOSP is $(\eps, \delta)$-DP. Moreover, with probability at least $1 - 2\zeta$, the output $\wpr$ of \cref{alg:meta pop} is a $\kappa$-second-order-stationary point of $F$, where \begin{align*}
 \kappa &\leq \wt{O}\left(\frac{(L \beta)^{1/3}}{n^{1/3}}\left[(LD)^{1/3}\left(\frac{d}{\eps n}\right)^{1/3}\right]\right) + \wt{O}\left(\left[L^4 \beta^3 D^3\left(\frac{d}{\eps n} + \sqrt{\frac{d}{n}} \right)^{3} \right]^{1/7}\left(\frac{\sqrt{d \ln(1/\delta)}}{\eps n}\right)^{3/7}\right)\\
 &+ \wt{O}\left(\frac{\beta}{\sqrt{\rho}}\left(\frac{1}{n\eps} + \frac{1}{\sqrt{n}} \right) \right)\Bigg[\left(\frac{L \beta}{n}\right)^{1/6}\left(LD\left(\frac{d}{\eps n} + \sqrt{\frac{d}{n}} \right)\right) \\
 &\;\;\; + \left(\frac{\sqrt{d \ln(1/\delta)}}{\eps n}\right)^{3/14}(L \beta^3)^{1/4}(LD)^{3/14} \left(\frac{d}{\eps n} + \sqrt{\frac{d}{n}} \right)^{3/14}\Bigg] \\
 &+ L\wt{O}\left(\frac{1}{\eps n} + \frac{1}{\sqrt{n}} \right). 
\end{align*}
\end{corollary}
\begin{proof}
Privacy follows from basic composition. 

The stationarity result is a consequence of \cref{thm: population meta} and \cref{lem: pop exp mech}. Namely, we use \cref{lem: pop exp mech} to plug $\psi = \wt{O}\left(LD \left(\frac{d}{\eps n} + \sqrt{\frac{d}{n}} \right) \right)$ into the expression for $\upsilon$ in \cref{thm: population meta}.
\end{proof}

Note that \cref{cor: appendix pop loss second order} immediately implies \cref{cor: pop loss second order stationary points}. 

\section{Improved Rates for Stationary Points of Non-Convex GLMs}
\label{app: GLM}

\begin{corollary}[Re-statement of \cref{cor: GLM}]
Let $f(w, (x,y))$ be a GLM loss function with $\beta, L, D = O(1)$.
Then, the JL method applied to the output of $\mathcal{M}=$~\cref{alg:meta} (with $\alg =$~Exponential Mechanism and $\mathcal{B}=$~DP-SPIDER) is $(\eps, \delta)$-DP and, given $n$ i.i.d. samples from $\mathcal{P}$, outputs $\wpr$ such that
\begin{align*}
    \expec\|\nabla F(\wpr)\| &\leq \wt{O}\left(\frac{1}{\sqrt{n}}\right) + \wt{O}\left(\frac{\sqrt{\trank}}{\eps n} \trank^{1/6} \wedge \frac{1}{(\eps n)^{3/7}} \right).
\end{align*}
\end{corollary}
\begin{proof}
The result is a direct consequence of \cref{lem: black box GLM} combined with \cref{cor: general nonconvex}. The fact that $\|\mathcal{M}(X)\| \leq poly(n, d, \beta, L, D)$ with high probability for $\mathcal{M}=$~\cref{alg:meta} (with $\alg =$~Exponential Mechanism and $\mathcal{B}=$~DP-SPIDER) follows from the proof of \citep[Corollary 6.2]{arora2022faster}, which showed that $\|\mathcal{B}(X)\|\leq poly(n, d, \beta, L, D)$ for any initialization $w_0$. 
\end{proof}

\section{Hyperparameters for Experiments}
\label{app: hyperparameters}

We tuned hyperparameters using the code at \url{https://github.com/lowya/How-to-Make-the-Gradients-Small-Privately/tree/main}. 

The ``optimal'' hyperparameters that we obtained for each algorithm and each value of $\eps$ are listed below (using $10$ independent epednent runs of the hyperparameter tuning code with fresh validation data in each run): 

\textbf{$\eps = 0.1$}
\begin{itemize}
    \item $T_1 = 50$ 
    \item SPIDER $q = 10$
    \item Warm-start $q = 100$
    \item  SGD $\eta =  0.0005$
    \item SPIDER $\eta = 0.005$
    \item Warm-start $\eta_{sgd} = 0.0005$
    \item Warm-start $\eta_{spider} = 0.005$
    \item Warm-start $\eps_1 = \eps/2$
\end{itemize}

\textbf{$\eps = 0.25$}
\begin{itemize}
    \item $T_1 = 50$ 
    \item SPIDER $q = 5$
    \item Warm-start $q = 5$
    \item  SGD $\eta =  0.0005$
    \item SPIDER $\eta = 0.001$
    \item Warm-start $\eta_{sgd} = 0.05$
    \item Warm-start $\eta_{spider} = 0.0005$
    \item Warm-start $\eps_1 = \eps/4$
\end{itemize}

\textbf{$\eps = 1$}
\begin{itemize}
    \item $T_1 = 1$ 
    \item SPIDER $q = 10$
    \item Warm-start $q = 10$
    \item  SGD $\eta =  0.0025$
    \item SPIDER $\eta = 0.0025$
    \item Warm-start $\eta_{sgd} = 0.001$
    \item Warm-start $\eta_{spider} = 0.0005$
    \item Warm-start $\eps_1 = \eps/4$
\end{itemize}

\textbf{$\eps = 2$}
\begin{itemize}
    \item $T_1 = 50$ 
    \item SPIDER $q = 5$
    \item Warm-start $q = 5$
    \item  SGD $\eta =  0.0025$
    \item SPIDER $\eta = 0.0025$
    \item Warm-start $\eta_{sgd} = 0.0025$
    \item Warm-start $\eta_{spider} = 0.0025$
    \item Warm-start $\eps_1 = \eps/4$
\end{itemize}

\textbf{$\eps = 4$}
\begin{itemize}
    \item $T_1 = 25$ 
    \item SPIDER $q = 5$
    \item Warm-start $q = 5$
    \item  SGD $\eta =  0.005$
    \item SPIDER $\eta = 0.005$
    \item Warm-start $\eta_{sgd} = 0.005$
    \item Warm-start $\eta_{spider} = 0.005$
    \item Warm-start $\eps_1 = \eps/100$
\end{itemize}

\end{document}